\documentclass{article} % For LaTeX2e
\usepackage{iclr2022_conference}
\usepackage{times}
\usepackage{booktabs}
\usepackage{amsfonts}
\usepackage{algorithm}
\usepackage{algorithmic}
\usepackage[pdftex]{graphicx}
\usepackage{amssymb}
\usepackage{amsmath}
\usepackage{subfig}
\usepackage{bm}
\usepackage{tikz}
\usepackage{dsfont}
\usepackage{verbatim}
\usepackage{natbib}
\usepackage{wrapfig}
\usepackage{lipsum}

\newcommand{\yao}[1]{{\color{orange}{\bf\sf [yao: #1]}}}

\newcommand{\junz}[1]{{\color{red}{[zj: #1]}}}
\newtheorem{theorem}{Theorem}

\newtheorem{definition}{Definition} 
\newenvironment{proof}{{\noindent\it Proof.}}{\hfill $\square$\par}
\newcommand*{\dif}{\mathop{}\!\mathrm{d}}
\newcommand*{\eref}[1]{(\ref{#1})}

\usepackage{amsmath,amsfonts,bm}

\def\eqref#1{equation~\ref{#1}}
\def\1{\bm{1}}

\DeclareMathAlphabet{\mathsfit}{\encodingdefault}{\sfdefault}{m}{sl}
\SetMathAlphabet{\mathsfit}{bold}{\encodingdefault}{\sfdefault}{bx}{n}

\usepackage{hyperref}
\usepackage{url}

\title{Model-based Reinforcement Learning with a Hamiltonian Canonical ODE Network}

\author{Yao Feng, Yuhong Jiang, Hang Su, Dong Yan, Jun Zhu\\
Tsinghua University\\
}

\iclrfinalcopy % Uncomment for camera-ready version, but NOT for submission.
\begin{document}

\maketitle

\begin{abstract}
% \junz{I read the main technical part, it is still confusing/unclear on what physical knowledge we're using, why such knowledge is important ... Maybe we should just say: We present a new model-based RL method that use Hamiltonian canonical ODE to model the physical world; and it provides an interface to easily consider physical knowledge (if any)? The current results can't convincingly support the argument of physical knowldge...}
% \yao{going to change the expression, PAC-bayes part deleted}
Model-based reinforcement learning usually suffers from a high sample complexity in training the world model, especially for the environments with complex dynamics. To make the training for general physical environments more efficient, we introduce Hamiltonian canonical ordinary differential equations into the learning process, which inspires a novel model of neural ordinary differential auto-encoder (NODA). NODA can model the physical world by nature and is flexible to impose Hamiltonian mechanics (e.g., the dimension of the physical equations) which can further accelerate training of the environment models. It can consequentially empower an RL agent with the robust
extrapolation using a small amount of samples as well as the guarantee on the physical plausibility. Theoretically, we prove that NODA has uniform bounds for multi-step transition errors and value errors under certain conditions. Extensive experiments show that NODA can learn the environment dynamics effectively with a high sample efficiency, making it possible to facilitate reinforcement learning agents at the early stage.   
%in learning both the model and the agent.
\end{abstract}

\section{Introduction}

Reinforcement learning has obtained substantial progress in both theoretical foundations~\citep{asadi2018lipschitz, jiang2018pac} and empirical applications~\citep{mnih2013playing, mnih2015human,peters2006policy, johannink2019residual}. In particular, model-free reinforcement learning (MFRL) can complete complex tasks such as Atari games~\citep{schrittwieser2020mastering} and robot control~\citep{roveda2020model}. 
However, the MFRL algorithms %are notoriously data inefficient. They 
often need a large amount of interactions with the environment~\citep{langlois2019benchmarking} in order to train an agent, which impedes their further applications. 
Model-based reinforcement learning (MBRL) methods can alleviate this issue by resorting to a model to characterize the environmental dynamics and conduct planning~\citep{van2019use,moerland2020framework}. 

In general, MBRL can quench the thirst of massive amounts of real data that may be costly to acquire, by using rollouts from the model~\citep{langlois2019benchmarking, deisenroth2011pilco}. It has witnessed numerous works on approximating the model with various strategies, such as
least-squares temporal difference \citep{boyan1999least}, guided policy search (GPS) \citep{levine2014learning}, dynamic Bayesian networks (DBN)~\citep{hester2012learning}, 
and deep neural networks~\citep{fujimoto2018addressing}. %In general, it is appealing for reinforcement learning, with which the agent can learn policies from the model~\citep{shen2020model}, encourage exploration~\citep{osband2013more}, and predict the long-horizon behaviors~\citep{hafner2019dream}. 
However, the sample efficiency of MBRL can still be limited due to the high sample complexity of learning a world model when the environment is complex. Traditional methods such as the Gaussian Processes based method~\citep{deisenroth2011pilco} can perform well on some problems with high sample efficiency, but they are not easy to scale to high-dimensional
problems~\citep{plaat2020model}. %\junz{by scalable, you mean "to high-dimensions", right? It's a bit confusing to the common use "scale up to large size data".}\yao{fixed}
High-capacity models scale well, but they often have low sample efficiency~\citep{plaat2020model}. The trade-off between scalability and sample complexity remains as a critical issue for model-based RL.

To address the aforementioned issue, we propose to introduce physical knowledge to reduce the sample complexity for learning high-dimensional dynamics in physical environments. We focus on reinforcement learning in an environment whose dynamics can be formulated by Hamiltonian canonical equations~\citep{goldstein2002classical}. Up till now, Hamiltonian dynamics have been successfully applied in numerous areas of physics from robotics to industrial automation. 
Specifically, we formulate the environments dynamics as ordinary differential equations (ODEs), and then use a novel network architecture called Neural Ordinary Differential Auto-encoder (NODA) as our world model, which is naturally induced by physical equations.

In particular, NODA consists of two parts --- an auto-encoder and an ODE network. We use the auto-encoder to get the underlying physical variables, and use the ODE network to learn the dynamics over physical variables. By using NODA, we can enjoy its ability of modeling the physical world as well as its flexibility of combining physical knowledge (if any), such as the dimension of physical variables. Theoretically, we provide uniform bounds for both multi-step transition errors and value errors for NODA by extending the former study of Lipschitz models \citep{asadi2018lipschitz} to continuous action spaces. It is noted that NODA can be combined with both MFRL methods like SAC~\citep{haarnoja2018soft} and MBRL methods like Dreamer~\citep{hafner2019dream} by facilitating the learning of the world model.  
Extensive experiments show that we can learn NODA well using a small number of data with an appropriate structure encoded, which can boost the sample efficiency by using imaginary trajectories over the environment models~\citep{todorov2012mujoco, schulman2015high}.

\section{Background}

We start by presenting the background knowledge of reinforcement learning, and then explain the relationship between MBRL and Hamiltonian mechanics. 
%We also introduce Lipschitz continuity and neural ODEs.

\subsection{Model-based Reinforcement Learning}

We consider the Markov decision process (MDP) model for reinforcement learning. Formally, an MDP is formulated as a tuple $\left\langle\mathcal{S}, \mathcal{A}, T, R, \gamma\right\rangle$, where $\mathcal{S}$ is the state space, $\mathcal{A}$ is the action space, $T: \mathcal{S}\times\mathcal{A}\rightarrow \mathbb{P}(\mathcal{S})$ is the transition function, $R: \mathcal{S}\times\mathcal{A}\rightarrow\mathbb{R}$ is the reward function, and  $\gamma\in[0, 1)$ is a discount factor.
We denote the set $\mathbb{P}(\cdot)$ as all probability measures on the space in the bracket. Our goal is to find a policy $\pi$ that can choose an action to maximize the accumulated reward. Here we focus on the challenging tasks with continuous state and action spaces (i.e., there are infinite states and actions).

MBRL aims to learn a policy by integrating planning with the aid of a known or learned model~\citep{moerland2020model}, and an essential part of MBRL is to learn a transition function that characterizes the environment. The transition function above is defined over a given state, but we can generalize it to represent the transition from a state distribution $z\in\mathbb P(\mathcal{S})$. By calling the generalized transition function recursively, we can get the $n$-step generalized transition function, which is defined as:
%
%\yao{compress, 1/2 sentence and go on to MBRL} 
%\hangx{we may give a big picture of model-based RL that integrate the learning and planning. and different from AlphaGo/AlphaZero which has a known model, we focus on the Hamiltonian systems in which it requires to learn a model fixed} 
%Many MBRL models like AlphaGo \citep{silver2017mastering} assume a known model, so they only need to do planning. However, here we firstly learn a model, and then use it for planning, which is more generic.
%
\begin{definition}[Transition Functions]\label{def: Transition functions}
In a metric state space $(\mathcal{S}, d_\mathcal{S})$ and an action space $\mathcal{A}$, we can define the generalized transition function of \,$T_{\mathcal G}$ (over state distribution $z$), and the $n$-step generalized transition function of\, $T_{\mathcal{G}}^{n}$ (for fixed sequence of actions) \citep{asadi2018lipschitz} as 
\begin{equation}
\begin{split}
    T_{\mathcal G}\left(\bm s^{\prime} \mid \bm s, \bm a\right) & = \int T\left(\bm s^{\prime} \mid \bm s, \bm a\right)z(\bm s)d\bm s \\
    T_{\mathcal{G}}^{n}(\cdot \mid z) & = \underbrace{T_{\mathcal{G}}\left(\cdot \mid  \cdots T_{\mathcal{G}}\left(\cdot \mid z, \bm a_{0}\right) \cdots, \bm a_{n-1}\right)}_{n \text{ recursive calls}}.
\end{split}\label{eq: Transition functions}
\end{equation}
\end{definition}
Here the generalized transition gives the distribution of outcome under a certain state distribution. For MBRL, it is nontrivial to learn the transition function (i.e., the dynamics for a physical environment) because $\mathcal{S}$ can be high-dimensional. 
Several attempts have been made to learn such dynamics, while they have various limitations. For instance, 
%iterative linear quadratic-gaussian (iLQG)~\citep{tassa2012synthesis} assumes that the transition function is known.
%\junz{if the function is known, is this an example of "learn such dynamics"?}\yao{deleted}
probabilistic inference for learning control (PILCO)~\citep{deisenroth2011pilco} uses Gaussian processes to model the transition function, but the inference part does not scale well to high dimensions~\citep{langlois2019benchmarking}.
Stochastic ensemble value expansion (STEVE)~\citep{buckman2018sample} and adaptation augmented
model-based policy optimization (AMPO)~\citep{shen2020model} use scalable machine learning models, but there is no further discussion on how to learn such a model efficiently. Actually, high-dimensional state and action spaces usually require much more data samples~\citep{plaat2020model}. Monte Carlo tree search (MCTS)~\citep{silver2017mastering} can introduce human knowledge of the transition function to enhance learning, but it is restricted to cases where transition functions are totally known. %Our proposed model provides a balance between the model complexity and data efficiency since the parameter size is small by incorporating a proper Hamiltonian prior.   
 %Linear models are simple, but their capacity is low; complex non-linear models have larger capacity, but they need more samples to converge. 
% However, NODA achieves a balance: it has the complexity needed to learn the dynamics, and it does not need a lot of data to train because the number of parameters is relatively small, because the form of its target (Hamiltonian canonical equations) is compact. Additionally, its physical meanings make it possible to use prior knowledge in order to learn dynamics by a light neural network.

%\yao{how MBRL do, summarize how they use transition function}

%\hang{A (deep) neural network dynamics model is typically factorized in three parts: i) an encoding function zt = fφenc(st), which maps the observation to a latent representation zt, ii) a latent dynamics function zt+1 = fφtrans(zt; at), which transitions to the next latent state based on the chosen action, and iii) a decoder function st+1 = fφdec(zt+1), which maps the latent state back to the next state prediction. This structure, visualized in Figure 6 (item 4), reminds of an auto-encoder (with added latent dynamics), as frequently used for representation learning in the deep learning community}

\subsection{Hamiltonian Mechanics}
\label{sec: Analytical Mechanics}
Methods of analytical mechanics have been introduced to predict the evolution of dynamic systems such as pendulums. For example, Lagrangian neural networks \citep{lutter2018deep, cranmer2020lagrangian} and Hamiltonian neural networks \citep{greydanus2019hamiltonian} can be used to simulate dynamic systems. These papers focus on how to model the Lagrangian or the Hamiltonian.
%Besides, the form of kinetic energy is fixed in Deep Lagrangian Networks \citep{lutter2018deep}, which restricts further applications.
One challenge for such methods is that 
%Our work does not explicitly model the Lagrangian or the Hamiltonian, because 
the equations for the Lagrangian or the Hamiltonian are second-order differential equations, which are difficult to model in a general way. Besides, numerical solutions of second-order differential equations are prone to error accumulation. 
%\junz{add the difficulty to motivate ours}\yao{fixed}

One natural idea is to reformulate second-order differential equations into first-order ones. Then, we can use an existing neural network~\citep{chen2018neural} to model these equations in the ODE form. In this paper, we concentrate on the Hamiltonian case, for which the first-order representation corresponds to Hamiltonian canonical equations \citep{junkins2009analytical}.\footnote{Hamiltonian canonical equations are widely used with clear physical meanings and more symmetric forms than the Lagrangian case.}
%We do not use the correspondence in the Lagrangian case, since Hamiltonian canonical equations are widely used with clear physical meanings and more symmetric forms.
Specifically, in Hamiltonian mechanics, we use pairs of generalized coordinate and generalized momentum $(q_k, p_k)$ to completely describe a dynamic system ($k\in\{0, 1, \cdots, K\}$), where $K$ is the dimension of generalized coordinates.It is noted that $K$ can be intuitively interpreted as the degree of freedom. We denote $p_k$ and $q_k$ as canonical states, and they are minimal independent parameters which can describe the state of the system. We define Hamiltonian $\mathcal{H}: \mathbb{R}^{2K+1}\rightarrow\mathbb{R}$ as a function of these variables and time $t$. Then the evolution of the system satisfies Hamiltonian canonical equations \citep{junkins2009analytical}:
\begin{equation}
    \begin{split}
        \frac{\dif q_k}{\dif t} & = \frac{\partial\mathcal{H}}{\partial p_k}, \\
        \frac{\dif p_k}{\dif t} & = -\frac{\partial\mathcal{H}}{\partial q_k} + Q_k(t),
    \end{split}
    \label{eq: Hamiltonian canonical equations}
\end{equation}
where $k\in\{0, 1, \cdots, K\}$, and $Q_k(t)$ are the generalized forces which describe the effects of external forces. 

One advantage of using these equations is that they can describe general dynamic systems. %\junz{two "genera" in one sentence, looks confusing, rewrite.}\yao{fixed}. %Canonical states can be any kinds of physical quantity, although they are denoted as coordinate or momentum. %\junz{any references that incorporate prior through Hamiltonian equations?}\yao{fixed}
Moreover, it is possible to incorporate prior knowledge into Hamiltonian canonical equations, e.g., by assuming a specific form of energy~\citep{sprangers2014reinforcement}.
In our case, prior knowledge can be available. For example, we may know the dimension $K$ of the generalized coordinate and generalized momentum, which only requires the knowledge of the 'degree of freedom' for a given system but makes a difference in training. Another example is about transfer learning: If we already learn the underlying dynamics, we can combine the learnt dynamics with other modules to transfer the learnt knowledge to a variety of different tasks. We empirically examine the these advantages in the experimental section. 

\section{Methodology}

% \begin{algorithm}[tb]
%   \caption{NODA for Reinforcement Learning}
%   \hangx{not enough information in the algorithm}
%   \label{alg: NODA}
% \begin{algorithmic}
%   \STATE {\bfseries Input:} any algorithm $H$, environment $env$, NODA model $m$
%   \STATE run algorithm H over $env$
%   \STATE using interactions with $env$ to train $m$
%   \IF{$m$ is trained well}
%   	\STATE run algorithm $H$ over $m$
%   \ENDIF
% \end{algorithmic}
% \end{algorithm}

\begin{figure}[t]
\centering
\includegraphics[height=6cm]{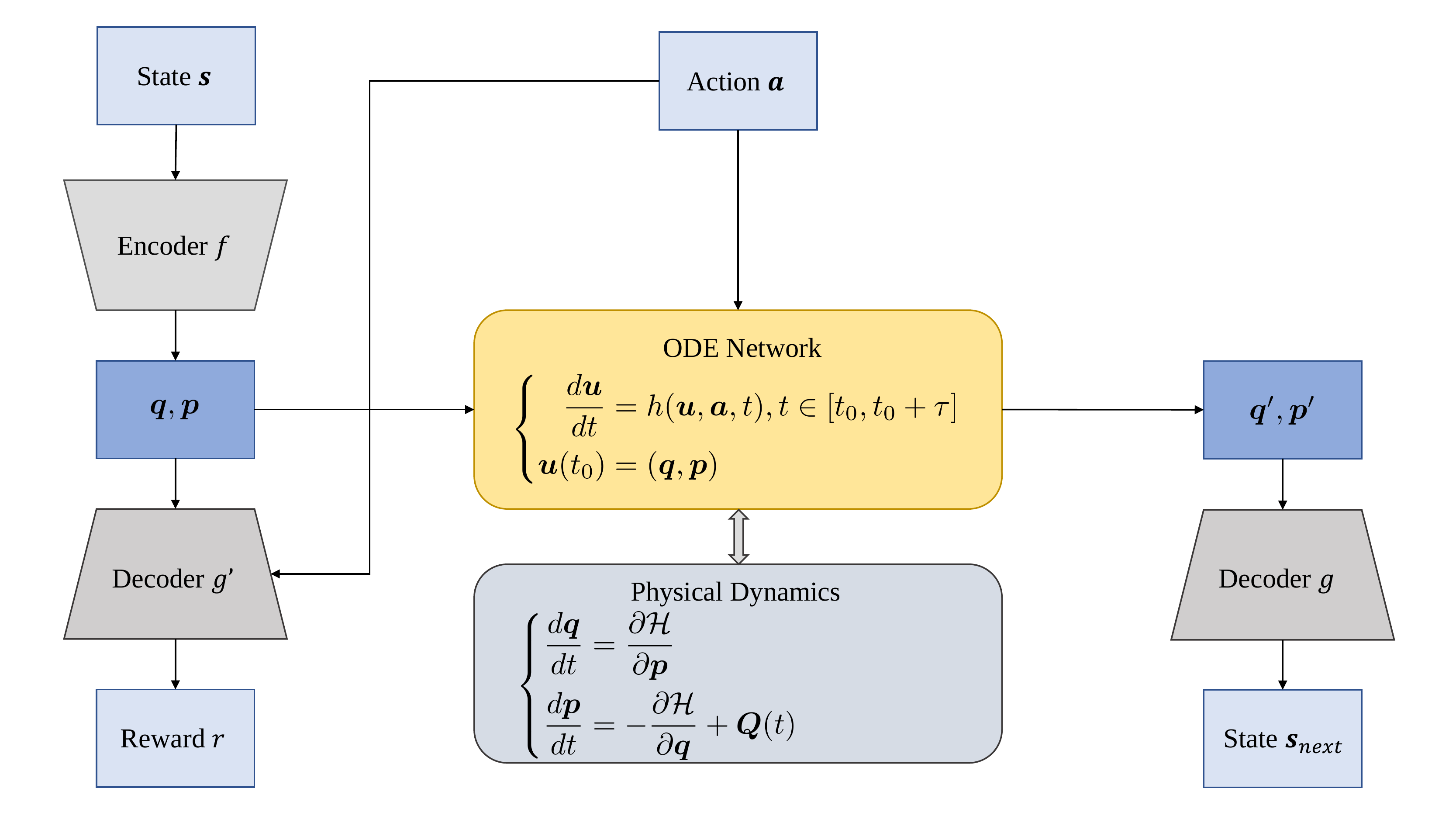}
\caption{NODA model structure. The state $\bm{s}$ goes through the encoder $f$ to get latent states $\bm q, \bm p$. The evolution of $\bm u = (\bm q, \bm p)$ defined by the ODE network is decoded by decoder $g$ to get the next state $\bm{s}_{next}$. The decoder $g'$ predicts reward $r$ using $\bm q$, $\bm p$ and action $\bm a$.}
\label{fig: NODA}
\end{figure}

%\junz{this is our key contribution: make sure every component is well explained, and the combination of multiple components is reasonable.}\yao{ok}

We now formally present NODA, which consists of an auto-encoder and an ODE network. NODA aims to serve as a simulator of the real environment (i.e., a dynamic system) by learning the transition function and the reward function. Here we assume that the transition is deterministic, otherwise an SDE network \citep{li2020scalable} can be used instead. Then we can use NODA to assist reinforcement learning by generating imaginary trajectories, as outlined in Algorithm~\ref{alg: NODA in RL}. Besides, we also discuss the prior knowledge that can be incorporated with NODA.

\subsection{Modeling Hamiltonian Canonical Equations}

We focus on modeling Hamiltonian canonical equations, since after getting these equations, we know the continuous evolution of the system. For many RL environments such as MuJoCo~\citep{todorov2012mujoco}, the discretization of the continuous evolution is just the transition function. However, it may be non-trivial to get these equations for a real dynamic system\footnote{We provide an example to show the derivation of these equations in Appendix A, where we derive these equations under the setting of Pendulum-v0 task in Gym~\citep{brockman2016openai}.}. So we propose to use neural ODE or ODE networks \citep{chen2018neural, chen2021eventfn}  to model these first-order differential equations.

Specifically, Neural ODE~\citep{chen2018neural} provides an elegant framework to model differential equations. It is defined by $\frac{\dif h(t)}{\dif t} = f(h(t), t, \theta)$. Here $f$ is a neural network with parameter $\theta$, and we use an integrator to transform the input $h(0)$ to output $h(\tau)$, where $\tau$ is a given time horizon. Neural ODE can be viewed as a continuous version of ResNet~\citep{he2016identity}. %Here we use this structure to learn Hamiltonian canonical equations.
Here we assume that the agent affects the environment through forces. So the generalized forces $\bm Q(t)$ in Eq.~\eref{eq: Hamiltonian canonical equations} correspond to the effects of action $\bm a$ in reinforcement learning, since generalized forces can describe all physical perturbations. 

%Instead of modeling the Lagrangian or the Hamiltonian, NODA directly focuses on modeling Hamiltonian canonical equations. Actually, Equation~\eref{eq: Hamiltonian canonical equations} are ordinary differential equations, which can be modeled by an ODE network \citep{chen2018neural}.
% Specifically, an ODE network uses a neural network (such as an MLP) to model the right hand side of Equation~\eref{eq: Hamiltonian canonical equations} (as a function of $(\bm q, \bm p, \bm Q)$), and integrates by numerical integration methods to get the evolved canonical states (solving $\dif\bm x/\dif t = \text{MLP}(t, \bm x)$).
% The ODE network can be applied to modeling general dynamic systems, such as mechanical arms and cars.

%In additional to higher computational costs, the Lipschitz condition of RNN is hard to guarantee, which is a burden of applying our theory.

% This straightforward modeling allows us to use some prior knowledge. If we know $K$, which is the number of pairs of generalized coordinate and generalized momentum, we can directly set the dimension $D$ of the ODE network to $2K$. In other words, the knowledge of the complexity of the system can guides the design of our model. Even if we only know the range of $K$, this knowledge can also reduce the search space of hyperparameter $D$. In this way, we can learn the simulator with more computational efficiency.

\subsection{NODA Model Structure}
Hamiltonian canonical equations describe the evolution over canonical states, but the observed state $\bm s$ is not necessarily composed of canonical states $\bm q, \bm p$. Assuming that our state $\bm s\in\mathcal{S}$ contains the complete information of the canonical states, we firstly construct an auto-encoder to map state $\bm s$ to canonical states $\bm q, \bm p$.\footnote{The auto-encoder is widely used in machine learning \citep{ballard1987modular, ng2011sparse, kingma2013auto}.}
Specifically, in the auto-encoder part, we use a function $f:\mathcal{S}\rightarrow\mathbb{R}^{2K}$ to get the canonical states. Here we denote the concatenation of $(\bm{q}, \bm{p})$ as $\bm{u}$, so $\bm u = f(\bm s)$. We also need to restore the state from canonical states, so we further define a function $g:\mathbb{R}^{2K}\rightarrow\mathcal{S}$. After going through the ODE network (defined later), we can decode the evolved canonical states $\bm s'$ to get the next state: $\bm{s}_{\text{next}} = g(\bm s')$. 

Since the canonical states are concise but complete to describe the system, the auto-encoder can serve as a dimension reducer to extract and refine state $\bm s$ when there is much redundancy in the state space. We further assume that the reward is related to the canonical states and the action. Assuming $\mathcal{A}\subset\mathbb{R}^m$, we can use another decoder $g': \mathbb{R}^{2K}\times\mathcal{A}\rightarrow\mathbb{R}$ to get predicted reward $r$. More specifically, $r = g'(\bm u, \bm a)$.
Then we let the predicted canonical states go through an ODE network to evolve.

 % In section~\ref{sec: Uniform Error Bounds for State Values}, we are going to show that reward shaping can be used to incorporate prior knowledge.

In the ODE network part, the encoded latent state $\bm u$ evolves through the ODE over a time interval $[t_0, t_0 + \tau]$. We can define the ODE network as a function
$\text{ODE}(h, \bm{u}, \bm{a}, t_0, t_0 + \tau)$, which returns the value of the solution of neural ODE at time $t_0 + \tau$. The neural ODE is described as:
\begin{equation}
%         \frac{\dif \bm{q}}{\dif t} &= \frac{\partial h_1}{\partial \bm{p}}, t\in [t_0, t_0 + T]
% \\
%         \frac{\dif \bm{p}}{\dif t} &= -\frac{\partial h_1}{\partial \bm{q}} + h_2(\bm a), t\in [t_0, t_0 + T]
% \\
    \frac{d\bm u}{dt} = h(\bm u, \bm a, t),~t\in [t_0,~t_0 + \tau],\\
    \label{eq: NODE}
\end{equation}
where $h$ is a neural network.

After combining the ODE network and the auto-encoder, we can get a full structure of NODA, which is shown in Figure~\ref{fig: NODA}. The inputs of NODA are the state $\bm s$ and the action $\bm a$, and the outputs of NODA are the predicted state $\bm{s}_{\text{next}}$ and the predicted reward $r$.  Our model is formally characterized as:
\begin{equation}
\begin{split}
        \bm{s}_{\text{next}} & = g(\text{ODE}(h, \bm u, \bm{a}, t_0, t_0 + \tau)),\\
        r & = g'(\bm u, \bm a).
\end{split}
    \label{eq: NODA}
\end{equation}
%
% \yao{only one equation for loss function}The loss function of NODA is composed of reconstruction loss and prediction loss. For reconstruction loss, we consider the distance between the original state $\bm{s}_0$ and the reconstructed state $\bm{s}_0'$, which is shown in Equation~\eref{eq: Reconstruction loss}:
% \begin{equation}
%     \mathcal{L}_{\text{recon}}(\bm{s}_0', \bm{s}_0) = .
%     \label{eq: Reconstruction loss}
% \end{equation}
% %\yao{use function form loss function}
% For prediction loss, we consider the distance between predicted next state $\bm{s}_T'$ and the actual next state $\bm{s}_T$ and the distance between predicted reward $r'$ and the actual reward $r$. The prediction loss is shown in Equation~\eref{eq: Prediction loss}:
% \begin{equation}
%     \mathcal{L}_{\text{pred}}(\bm{s}_T', \bm{s}_T, r', r) = \lambda_1 + .
%     \label{eq: Prediction loss}
% \end{equation}
To learn the unknown parameters in $f, g$ and $h$, we define the total loss function for NODA as a convex combination of the state loss and the reward loss:
\begin{equation}
    \mathcal{L} = \mu(||g(\bm u)-\bm{s}||_2^2 + ||\bm{s}_{\text{next}}- \bm{s}_{\text{next}}^*||_2^2) + (1 - \mu)||r- r^*||_2^2,
    \label{eq: Loss function}
\end{equation}
where $\mu\in(0, 1)$, and $*$ denotes the ground truth. %When training the RL agent, we use samples in the replay buffer to update our model at the same time by backpropagation and Adam~\citep{kingma2014adam}. More specifically, 
By minimizing this objective, we jointly train the RL agent and our model using interactions with the environment, and use our model together with the real environment to provide training data for the RL agent after certain training steps. More details are provided in Appendix C.

% In addition to giving a physical way to view former transition model in the latent space, NODA provides a continuous model for a dynamics system, which is enabled by Hamiltonian canonical equations and is usually not supported in former transition models. Actually, we can easily get the intermediate values of canonical states ($\bm{p}(t_0)$ and $\bm{q}(t_0)$) at time $t_0\in[t, t+\tau]$ by modifying the time interval of integration. After that, we can use decoder $g$ to get the intermediate state in the state space.

\begin{wrapfigure}{R}{0.54\textwidth}
    \begin{minipage}{0.54\textwidth}
      \begin{algorithm}[H]
         \caption{NODA for Reinforcement Learning}
  \label{alg: NODA in RL}
\begin{algorithmic}

  \STATE {\bfseries Input:} Reinforcement learning agent $s$, environment $env$, NODA model $m$
  \REPEAT
  \STATE Collect data from interactions with $env$
  \STATE Use interactions with $env$ to train $m$
  \STATE Use interactions with both $env$ and $m$ to train $s$
  \UNTIL{Certain number of steps}
  \RETURN Reinforcement learning agent $s$, NODA model $m$
\end{algorithmic}
\end{algorithm}
 \end{minipage}
  \end{wrapfigure}

As a model of Hamiltonian canonical equations, NODA can describe general dynamic systems, and the form of canonical states can be very general. Moreover, NODA corresponds with real dynamics of the environment, which allows us to use our knowledge of the Hamiltonian. For example, we can determine $2K$ with little physical knowledge, which is the dimension of the canonical state space of a dynamic system, and we demonstrate the effects of such dimension in the experimental part. For stronger prior knowledge, we can replace the ODE network by the derived form of Hamiltonian canonical equations, or do transfer learning between similar dynamic systems such as transferring the auto-encoder or the ODE network. We also demonstrate the effects of transfer learning in the experimental part. This makes it possible for us to make trade-offs between human expert knowledge and the required number of training samples.
%\junz{this is the most weak part as I said before. It's not clear what knowledge, why such knowledge is important, how to incorporate knowledge (if any)...}\yao{I tried to make it stronger}

%Actually, only by using NODA, we still introduce a prior knowledge or inductive bias that the environment is a dynamic system with a differentiable Hamiltonian. Correct prior knowledge makes training more efficient, which makes it possible to quickly learn a model to enhance sample efficiency for reinforcement learning.  \junz{this should be clear when you choose NODA to model dynamics; otherwise, move this part to where it is needed.}

%Another aspect is correctness. In section~\ref{sec: Theoretical Analysis}, we are going to give theoretical error bounds for NODA.

\section{Theoretical Analysis}
\label{sec: Theoretical Analysis}
%\yao{what we are going to prove, why. important.}
%\yao{compress, only motivation}In this section, we are going to give error bounds for both multi-step transition and values. In the transition part, we prove that the state transition in both dynamic systems and NODA is Lipschitz under some conditions. As a result, NODA can approximate the environment with uniform error bounds of multi-step transition.

%In the value part, we prove that the reward function of dynamic systems can be Lipschitz under some conditions. So, we can use a Lipschitz function $g'$ to model the reward function. In this way, we can give uniform error bounds for values.
In this section, we present error bounds for both multi-step transition and values for NODA.

\subsection{Lipschitz Dynamic Systems and NODA}
Former work \citep{asadi2018lipschitz} studied using Lipschitz function groups to model the transition of the environment. However, there are error bounds only if the real environment is Lipschitz, and that work did not mention what kind of environments have a Lipschitz transition. Theorem~\ref{th: Lipschitz dynamic systems} proves that under certain conditions, the transition functions of dynamic systems are Lipschitz, which is the basis of transition and state value error bounds. 

Before stating the theorem, we firstly define Lipschitz continuity, which measures the maximum magnitude of enlargement of the perturbation of input at the output side. For reinforcement learning, Definition~\ref{def: Lipschitz models} is the condition of uniformly Lipschitz continuous in the former work \citep{asadi2018lipschitz}, and it plays a central role in our theorems.

\begin{definition}[Lipschitz Models]
In a metric state space $(\mathcal{S}, d_\mathcal{S})$ and an action space $\mathcal{A}$, we say that $f$ is a Lipschitz model if the Lipschitz constant \citep{asadi2018lipschitz}
\begin{equation}
K_F:= \sup_{a\in \mathcal{A}}\sup_{\bm s_{1}, \bm s_{2} \in \mathcal{S}} \frac{d_{\mathcal{S}}\left(f\left(\bm s_{1}, \bm a\right), f\left(\bm s_{2}, \bm a\right)\right)}{d_{\mathcal{S}}\left(\bm s_{1}, \bm s_{2}\right)} < \infty.
\label{eq: Lipschitz models}
\end{equation}
\label{def: Lipschitz models}
\end{definition}
%
% Before the formal theorem and its proof, we are going to state an important assumption. Actually, Lipschitz conditions are very strong, especially in continuous spaces. So we define equivalences to make a continuous space correspond to a discrete space. Details are shown in Definition~\ref{def: equivalences}. 
% \begin{definition}
% \textbf{($\epsilon$-equivalent)} For a metric space $(X, d)$, it is $\epsilon$-equivalent if for all $x_1, x_2\in X$
% \begin{equation}
% d(x_1, x_2)=0 \text{ or } d(x_1, x_2)\geq\epsilon.
% \end{equation}
% \label{def: equivalences}
% \end{definition}
% For $\mathbb{R}^n$, it is easy to show that it is $10^m$-equivalent if all numbers are rounded to $m$-th digit. In fact, the values in our computer is discrete, which guarantees the $\epsilon$-equivalence even if we do not explicitly use any quotient space (but $\epsilon$ may be very small).
\begin{theorem}[Lipschitz Dynamic Systems]
For a dynamic system with a $C^2$ continuous Hamiltonian $\mathcal{H}: \mathbb{R}^{2K+1}\rightarrow\mathbb{R}$, if the state $\bm{s}$ is in a bounded closed set $\mathcal{S}\subset \mathbb{R}^{l}$, the evolution time equals $\tau$, the generalized force $Q_k$ is $C^1$ continuous with respect to states and bounded (for any dimension $k$), and the transformation from states to canonical states $f^*: \mathcal{S}\rightarrow\mathbb{R}^{2K}$ is Lipschitz, then the canonical states are Lipschitz with respect to time, and the environment with respect to canonical states is Lipschitz. Additionally, if the transformation from canonical states to states $g^*: \mathbb{R}^{2K}\rightarrow\mathcal{S}$ is Lipschitz, then the environment is Lipschitz, which means (here $\bm s\neq \bm s'$)
\begin{equation}
\sup_{a\in \mathcal{A}}\sup _{\bm s, \bm s' \in \mathcal{S}} \frac{d_{\mathcal{S}}\left(\bm s_{\text{next}}, \bm s_{\text{next}}'\right)}{d_{\mathcal{S}}\left(\bm s, \bm s'\right)} < \infty.
\label{eq: Dynamic systems as a Lipschitz model}
\end{equation}
\label{th: Lipschitz dynamic systems}
\end{theorem}
We provide a detailed proof in Appendix B, and here we only give a sketch. 

\begin{proof}
By noting that the composition of Lipschitz functions is Lipschitz, we only need to analyse the ODE part. The conditions of $\mathcal{H}$ and $Q_k$ guarantees that the right hand side of the ODE is Lipschitz, which leads to the Lipschitz condition in this part.
\end{proof}

Actually, the conditions in Theorem~\ref{th: Lipschitz dynamic systems} are not hard to satisfy. The Hamiltonian $\mathcal{H}$ itself already has derivatives with respect to $\bm{q}$, $\bm{p}$ and $t$. We just further assume that the Hamiltonian $\mathcal{H}$'s derivatives with respect to $\bm{q}$, $\bm{p}$ and $t$ is differentiable continuous. Actually, for many dynamic systems such as spring-mass systems or 
three-body systems, the Hamiltonian is $C^{\infty}$ continuous.
%In many cases, $\mathcal{H}$ does not contain $t$ explicitly (so $\frac{\partial\mathcal{H}}{\partial t} = 0$), which means energy conservation if there is no generalized force. Meanwhile, the Hamiltonian is the energy of the system. 
%A bounded closed state space and the Lipschitz conditions of $f^*$ and $g^*$ are also not strict conditions. %In the appendix, we give an example that the environment Pendulum-v0 in Gym \citep{brockman2016openai} satisfies the Lipschitz conditions. 

%One may argue that a small $\epsilon$ can lead to a large Lipschitz constant. Although $\epsilon$ may be small, it can be offset by a small $T$. Intuitively, we cannot require a dynamic system to be Lipschitz when a long time has elapsed. In the theory of ordinary differential equations, the existence and uniqueness of the solution is also only guaranteed in a small time range.

As an imitation of Hamiltonian canonical equations, it is natural to validate if the NODA model is Lipschitz. We analyse this in Theorem~\ref{th: Lipschitz NODA} in which the Lipschitz model will be one of the conditions in error bounds. 

\begin{theorem}[Lipschitz NODA]
For the NODA model, if the state $\bm{s}$ is in a bounded closed set $\mathcal{S}\subset \mathbb{R}^{l}$, $f: \mathcal{S}\rightarrow\mathbb{R}^{2K}$ is Lipschitz, the evolution time equals $\tau$, the action $a_m$ is $C^1$ continuous with respect to states and bounded (for any dimension $m$), function $h$ is $C^1$ continuous, then canonical states are Lipschitz with respect to time, and NODA with respect to canonical states is Lipschitz. Additionally, if the transformation from canonical states to states $g: \mathbb{R}^{2K}\rightarrow\mathcal{S}$ is Lipschitz, then NODA is Lipschitz, which means (here $\bm s\neq \bm s'$)
\begin{equation}
\sup_{\bm a\in \mathcal{A}}\sup _{\bm s, \bm s' \in \mathcal{S}} \frac{d_{\mathcal{S}}\left(\bm s_{\text{next}}, \bm s_{\text{next}}'\right)}{d_{\mathcal{S}}\left(\bm s, \bm s'\right)} < \infty.
\label{eq: NODAs as a Lipschitz model}
\end{equation}
\label{th: Lipschitz NODA}
\end{theorem}
The proof is in Appendix B with the similar idea as Theorem~\ref{th: Lipschitz dynamic systems}.

\begin{comment}
In experiments, we can easily control the Lipschitz constant for canonical states. In addition to the choice of $T$, a simple way is to clip the output of $h$. If we are not sure about the Lipschitz constant for real dynamics, we can just use a sufficiently large value $L>0$, and clip each dimension of $h$ to $[-L, L]$. On the contrary, if we have some prior knowledge of Hamiltonian canonical equations, we can use $L_k<L$ for each dimension to do clip, which serves as a regularization value intuitively.
\end{comment}

\subsection{Uniform Error Bounds for Multi-step Transition}

Here we firstly define a metric called Wasserstein Metric in Definition~\ref{def: Wasserstein metric}. Wasserstein metric describes how to move one distribution to another with the least cost. It has been applied to generative adversarial networks \citep{arjovsky2017wasserstein} and reinforcement learning \citep{asadi2018lipschitz}.
\begin{definition}[Wasserstein Metric]
In a metric space $(X, d)$, the Wasserstein
metric between two probability distributions $z_1$ and $z_2$ in $\mathbb{P}(X)$ is
\begin{equation}
W\left(z_{1}, z_{2}\right):=\inf _{j \in \Lambda} \iint j\left(\bm s_{1}, \bm s_{2}\right) d\left(\bm s_{1}, \bm s_{2}\right) d \bm s_{2} d \bm s_{1},
\label{eq: Wasserstein metric}
\end{equation}
where $\Lambda$ is a set of all joint distributions $j$ on $X\times X$ with marginal distributions $z_1$ and $z_2$ \citep{arjovsky2017wasserstein}.
\label{def: Wasserstein metric}
\end{definition}

Under the conditions of Theorem~\ref{th: Lipschitz dynamic systems} and Theorem~\ref{th: Lipschitz NODA}, the dynamic system and the NODA model are Lipschitz models. On that basis, we can give uniform error bounds for multi-step transition, which is shown in Theorem~\ref{th: Transition error bounds}. The theorem tells us that the multi-step transitions of the NODA model and the environment do not differ much under certain conditions, which gives support for using the NODA model as the imaginary environment.

\begin{theorem}[Transition Error Bounds]
Under the conditions in Theorem~\ref{th: Lipschitz dynamic systems} and Theorem~\ref{th: Lipschitz NODA}, we already know that the transition function $T_{\mathcal G}\left(\bm s^{\prime} \mid \bm s, \bm a\right)$ of the environment and the transition function $\widehat T_{\mathcal G}\left(\bm s^{\prime} \mid \bm s, \bm a\right)$ of the NODA model are Lipschitz. We denote the Lipschitz constants of these transition functions as $K_1$ and $K_2$, respectively. Let $\bar{K} = \min\{K_1, K_2\}$, then $\forall n\geq 1$:
\begin{equation}
\delta(n):=W\left(\widehat T_{\mathcal{G}}^{n}(\cdot \mid \mu), T_{\mathcal{G}}^{n}(\cdot \mid \mu)\right) \leq \Delta \sum_{i=0}^{n-1}(\bar{K})^{i},
\label{eq: Transition n-step bounds}
\end{equation}
where $\Delta$ is an upper bound of Wasserstein metric between $\widehat T\left(\cdot \mid \bm s, \bm a\right)$ and $T\left(\cdot \mid \bm s, \bm a\right)$. This upper bound is tight for linear and deterministic transitions.
\label{th: Transition error bounds}
\end{theorem}

The original theorem~\citep{asadi2018lipschitz} gives a a bound for a fixed action sequence. However, here our definitions of Lipschitz environments and models are uniform for all actions. So, by using the original proof, we give a same error bound for all possible action sequences. Thus, we get a uniform error bound under the continuous action space. This concludes the proof. 

\begin{theorem}
\textbf{(Value Error Bounds)} Under all the conditions in Theorem~\ref{th: Transition error bounds}, if the reward function $R(\bm s)$ is uniformly Lipschitz for any action, which means we can define
\begin{equation}
K_R:= \sup_{a\in \mathcal{A}} \sup _{\bm s_{1}, \bm s_{2} \in \mathcal{S}} \frac{\left|R\left(\bm s_{1}, \bm a\right)-R\left(\bm s_{2}, \bm a\right)\right|}{d_\mathcal{S}\left(\bm s_{1}, \bm s_{2}\right)} < \infty.
\label{eq: Lipschitz reward function}
\end{equation}
If we define state values as
\begin{equation}
V_{T}(\bm s):=\sum_{n=0}^{\infty} \gamma^{n} \int T_{\mathcal{G}}^{n}\left(\bm s^{\prime} \mid \delta_{\bm s}\right) R\left(\bm s^{\prime}\right) d \bm s^{\prime},
\label{eq: State values}
\end{equation}
where $\delta_{\bm s}$ means the probability that the state being $\bm s$ equals 1.
Then $\forall \bm s\in\mathcal{S}$ and $\bar{K}$ (defined in Theorem~\ref{th: Transition error bounds})$\in[0, \frac{1}{\gamma}]$, we have
\begin{equation}
\left|V_{T}(\bm s)-V_{\widehat{T}}(\bm s)\right| \leq \frac{\gamma K_{R} \Delta}{(1-\gamma)(1-\gamma \bar{K})}.
\label{eq: Value error bounds}
\end{equation}
\label{th: Value error bounds}
\end{theorem}
We put the theorem of the uniform error bounds of state values here, and defer its proof and more detailed discussions of theorems to Appendix B.

\section{Experiments}
In this section, we validate the efficiency of NODA's learning and its ability of boosting both MFRL and MBRL methods.
\subsection{Experimental Setup}
\paragraph{Baseline methods}
One natural baseline method for NODA is called AE, which means replacing the ODE network in NODA by an MLP. For physical simulation tasks, we choose AE and Hamiltonian neural network (HNN) \citep{greydanus2019hamiltonian} as baseline methods. For MFRL methods, we choose SAC and TD3 as our baselines. We modify the code from spinning up~\citep{SpinningUp2018} to use GPU in training. For MBRL methods, we compare our approach with state-of-the-art methods in the literature. Specifically, we focus on two methods on visual control tasks, Dreamer~\citep{hafner2019dream} and BIRD~\citep{zhu2020bridging}, as our MBRL baselines. We reproduce their results by their released codes respectively.

\paragraph{Implementation}
We implement NODA mainly by using pytorch~\citep{NEURIPS2019_9015} and torchdiffeq (https://github.com/rtqichen/torchdiffeq) as the implementation of the ODE network~\citep{chen2018neural, chen2021eventfn}. We integrate NODA with MFRL methods by using it to generate imaginary trajectories and compare with MFRL baselines. When comparing with the two MBRL baselines, we combine NODA with Dreamer to get our approach: NODA-Dreamer, and we implement this in TensorFlow as Dreamer does. We use tfdiffeq (https://github.com/titu1994/tfdiffeq) as the implementation of the ODE network, which runs entirely on Tensorflow Eager Execution. We refer to Appendix C for more implementation details.
%\footnote{\url{https://github.com/titu1994/tfdiffeq}}
%In section~\ref{sec: simulation effectiveness}, each experiment is run on a single GeForce GTX TITAN X. In section~\ref{sec: sample efficiency}, each TD3/SAC/AE-SAC/NODA-SAC experiment is run on a single GeForce RTX 2080 Ti, while each Dreamer/BIRD/NODA-Dreamer experiment is run on a single A100-SXM4-40GB.

\paragraph{Tasks} For physical simulation tasks, we use the setting of pixel pendulum and real pendulum in the paper of HNN~\citep{greydanus2019hamiltonian}. The pixel pendulum task aims at predicting the next frame of the pendulum using former adjacent frames, and the real pendulum task does the same thing using physical values as input. Besides, we use NODA to learn the transition part in the Ant-v3 task (actions are sampled randomly) in Gym~\citep{brockman2016openai} to study the parameter sensitivity and transferability of NODA. For MFRL methods, we choose 4 MuJoCo environments in Gym to compare the performance of TD3, SAC, AE-SAC and NODA-SAC. For MBRL methods, since the two baselines themselves are evaluated on DeepMind Control Suite (DMC)~\citep{tassa2018deepmind}, we also evaluate NODA-Dreamer on 4 environments in DMC, in order to make a fair comparison with them. More experimental details can be found in Appendix D.

\subsection{Simulation Effectiveness}
\label{sec: simulation effectiveness}
%To see whether NODA can efficiently model the dynamics, we compare NODA, AE and Hamiltonian neural network (HNN) \citep{greydanus2019hamiltonian}. Here we slightly change the loss function of NODA in order to predict the velocity instead of the next state (more details are in the appendix).
%Since the target of Lagrangian neural networks \citep{cranmer2020lagrangian} is the acceleration, which is different from our model, we do not include these neural networks. 

The testing loss curves of HNN, AE and NODA over two physical environments are shown in Figure~\ref{fig: simulation effectiveness and sensitivity} (a) and (b). As the training loss functions among models can be different, we choose to compare the well-defined testing loss. In each experiment, the number of NODA's parameters equals to that of AE's parameters, and it is no more than that of HNN's parameters. In both tasks, NODA converges quickly, and achieves the best testing loss over only thousands of training iterations. It suggests that the inductive bias introduced by NODA that the system obeys Hamiltonian canonical equations accelerates training.% More details are in the appendix.

\begin{comment}
\begin{table}[t]
    \caption{Simulation effectiveness}
    \label{tab: simulation effectiveness}
    \centering
    \begin{tabular}{cccc}
    \toprule
    \multicolumn{4}{c}{Task: Pixel Pendulum}\\
    \midrule
    Model  & $\mathcal{L}_\text{Training}$ & $\mathcal{L}_\text{Testing}$ & $\text{N}_{\text{Params}}$\\
    \midrule
    HNN     &   1.2E-7±2.9E-10    &   1.2E-7±5.9E-10 & 13682\\
    MLP+AE      &   1.0E-7±2.6E-10    &  1.0E-7±5.3E-10  & 13682\\
    \midrule
    NODA      &   \textbf{1.0E-7±2.7E-10}    &   \textbf{9.9E-8±5.4E-10} & 13682\\
    \midrule
    \multicolumn{4}{c}{Task: 3 Body}\\
    \midrule
    HNN     &   7.5E-2±2.0E-2    &   4.1E-1±1.7E-1 & 43200\\
    MLP+AE      &   1.2E-1±2.2E-2    &  4.4E-1±1.5E-1  & 38424\\
    \midrule
    NODA      &   \textbf{4.5E-2±1.2E-2}    &   \textbf{3.0E-1±1.4E-1} & 38424\\
    \bottomrule
    \end{tabular}
\end{table}
\end{comment}

\begin{figure}[t]
%\centering
\subfloat[]{
\includegraphics[width=0.25\linewidth]{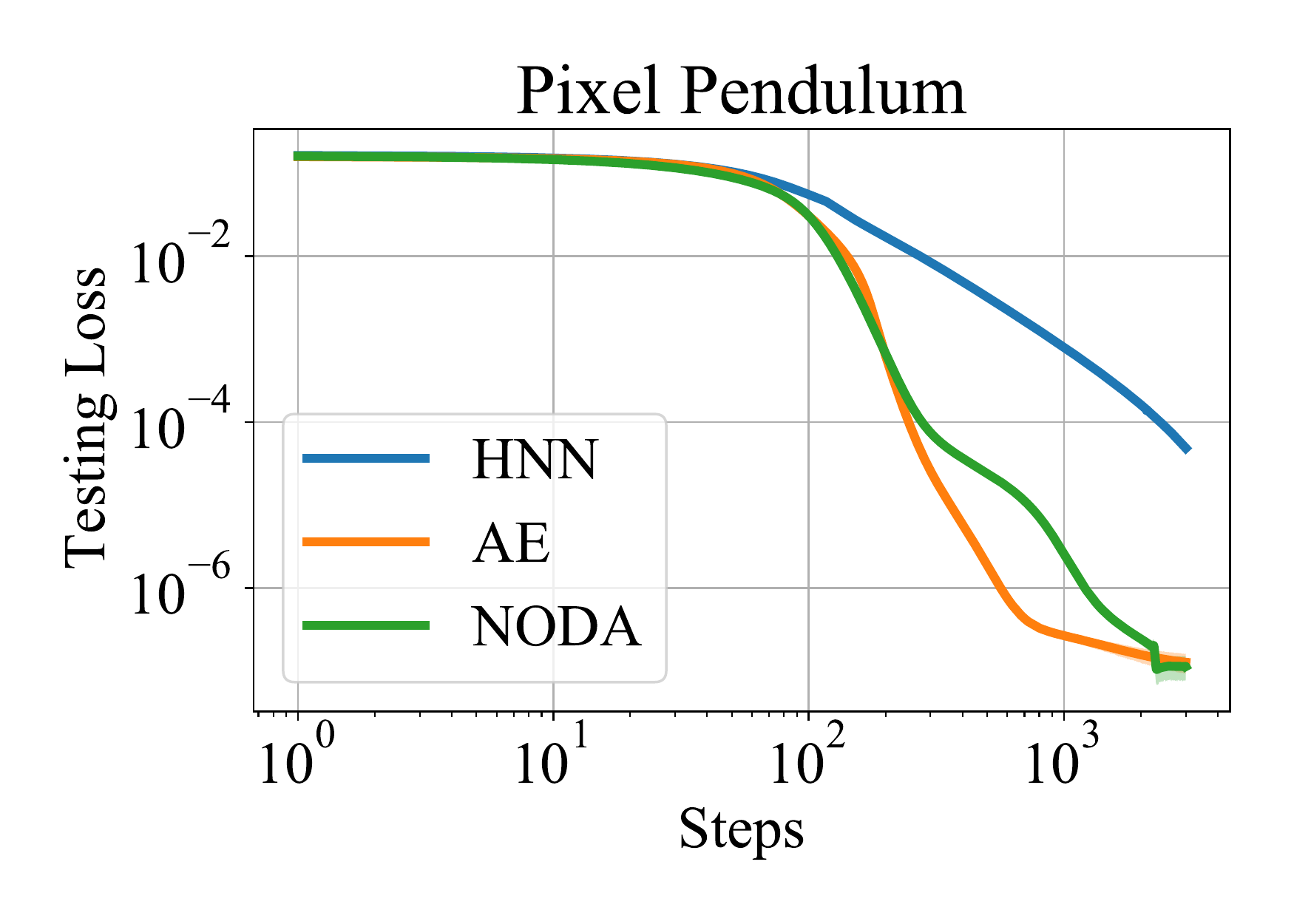}}
\subfloat[]{
\includegraphics[width=0.25\linewidth]{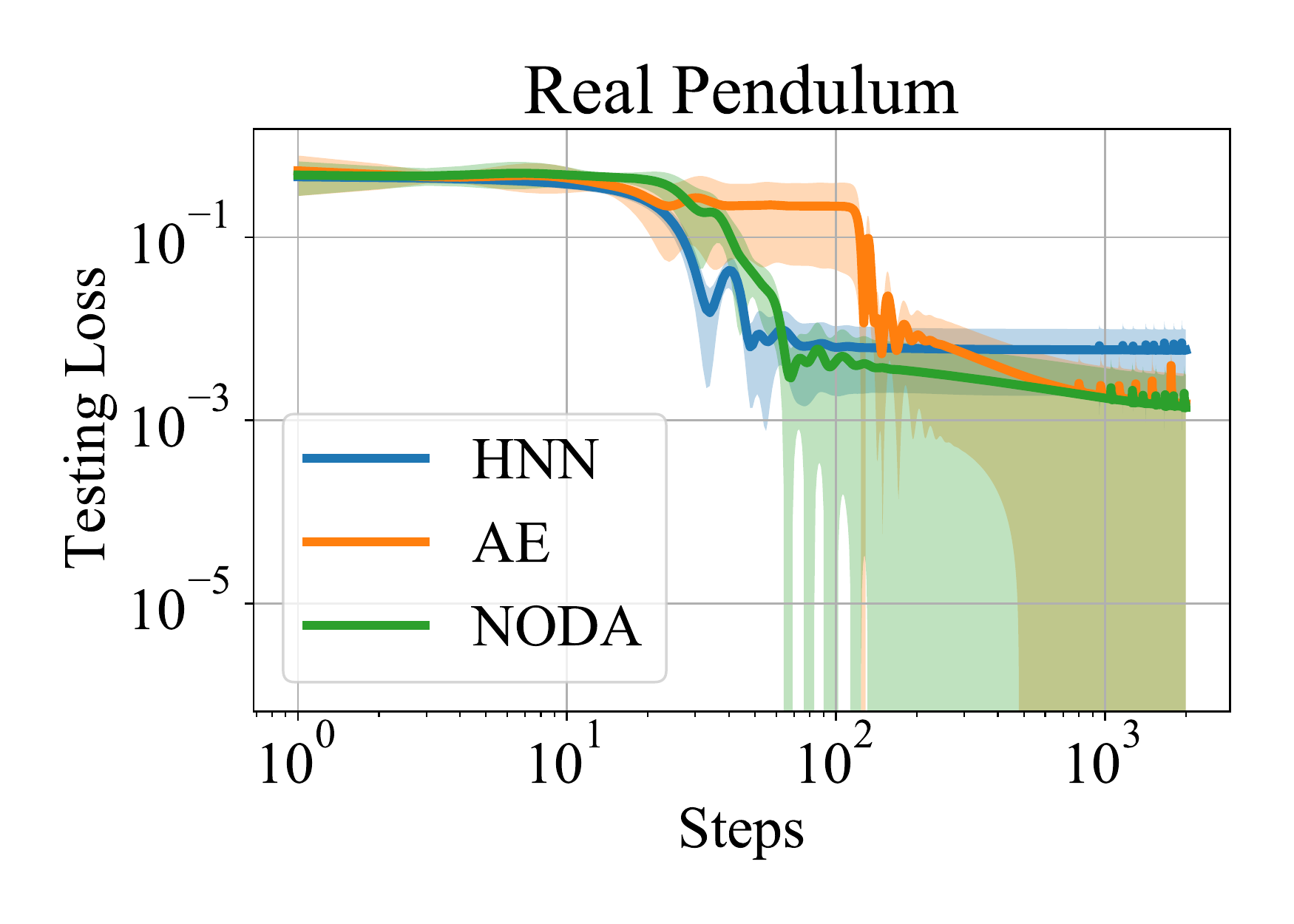}}
\subfloat[]{
\includegraphics[width=0.25\linewidth]{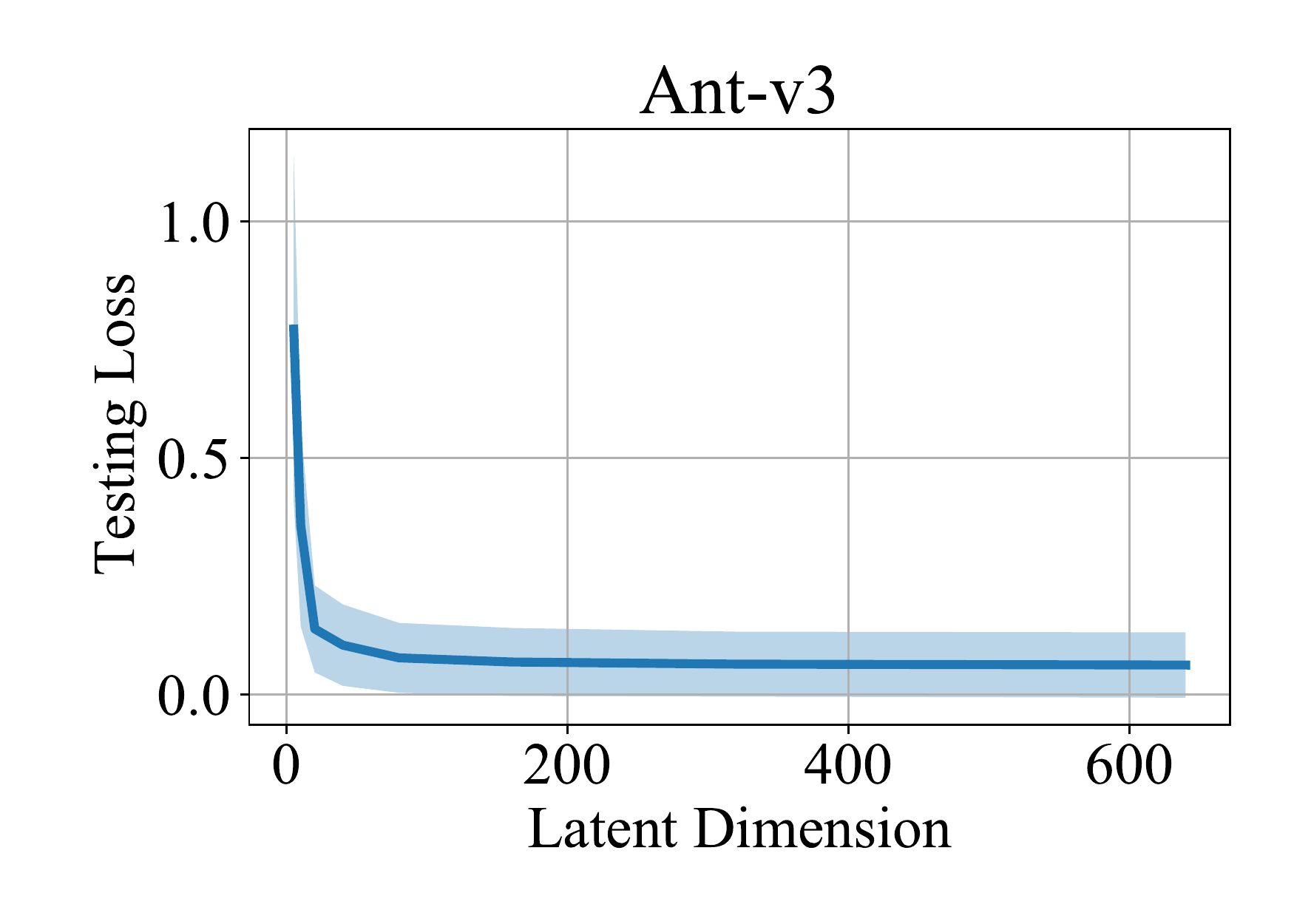}}
\subfloat[]{
\includegraphics[width=0.25\linewidth]{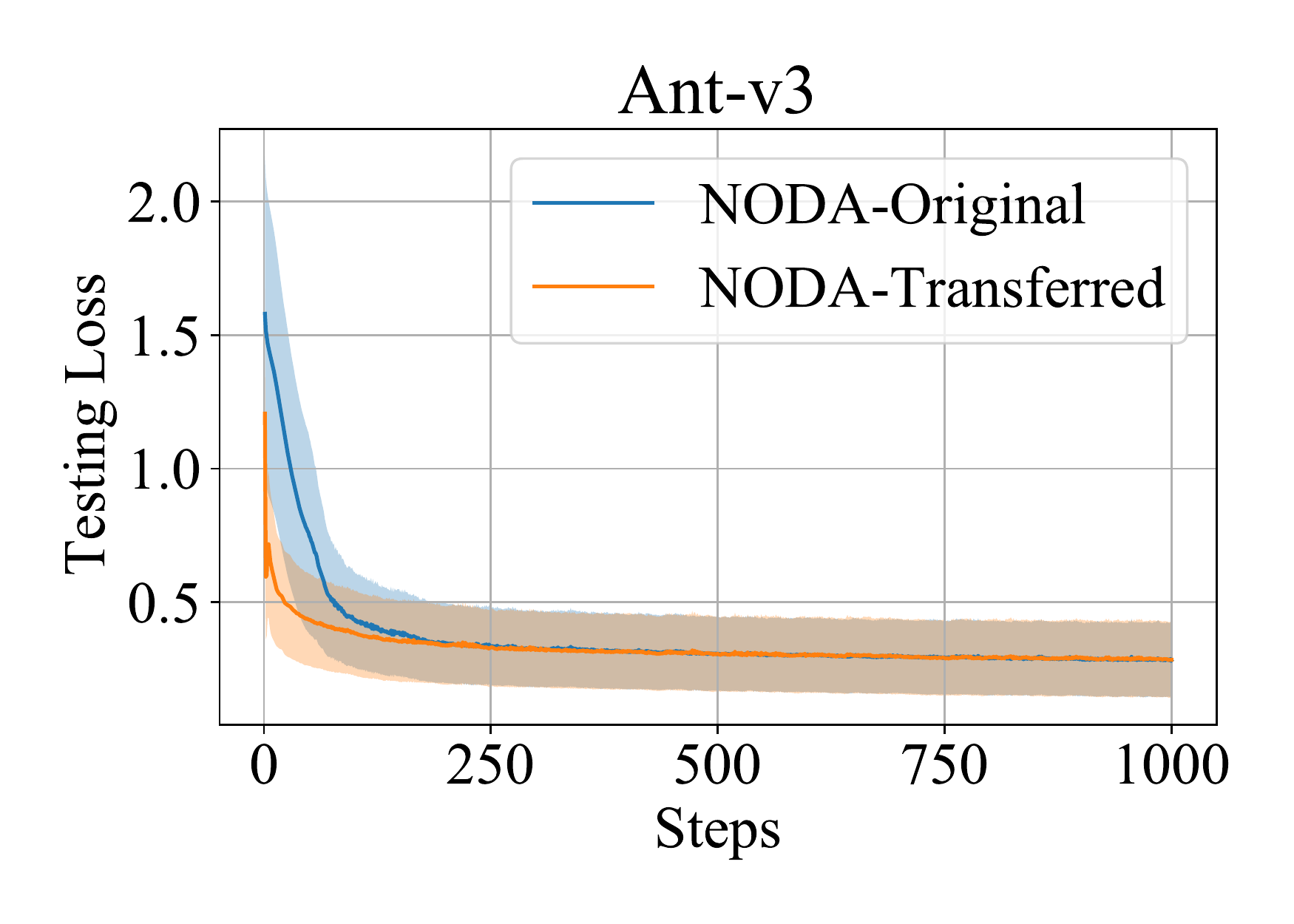}}
%\hspace{0.08\linewidth}
\caption{(a, b): Results of NODA and other physical simulators in modeling dynamic systems. (c): Sensitivity of NODA's latent dimension when learning the transition in the Ant-v3 task. (d):  Transferability of NODA in the Ant-v3 task.}
\vspace{-1.5ex}
\label{fig: simulation effectiveness and sensitivity}
\end{figure}
\begin{figure}[b]
\centering
\includegraphics[width=\linewidth, trim=250 0 250 0, clip]{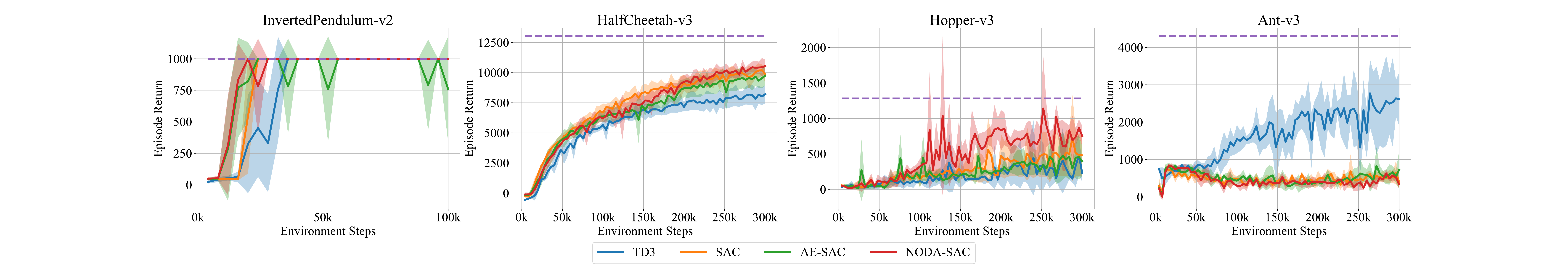}
\caption{Results of TD3, SAC, AE-SAC and NODA-SAC on MuJoCo environments in Gym. The models are evaluated every 4k steps for 10 episodes, and the means and standard deviations are computed across 10 episodes across 4 seeds. The dashed lines are the best asymptotic performances of 1M steps.
}
\vspace{-1.5ex}
\label{fig: training process of sample Efficiency part}
\end{figure}
Figure~\ref{fig: simulation effectiveness and sensitivity} (c) shows the effect of the latent dimension on NODA's testing error after 3,000 iterations. Here we use NODA to learn the transition part in the Ant-v3 task. By mechanics, we know that the dimension of Hamiltonian canonical equations (corresponding to the dimension of the latent space in NODA) is approximately 40 in this task. We can see that a latent dimension of 40 is sufficient for NODA to achieve a high performance. This suggests that if we know the approximate dimension, we can use the prior knowledge to get a sweet point.

Figure~\ref{fig: simulation effectiveness and sensitivity} (d) shows the testing loss of the original NODA and the transferred NODA in the Ant-v3 task. We train the latter model for 100 steps over the one step transition of Ant-v3, and the final task is to predict the two-step transition. The figure suggests that we can transfer NODA over similar environments to accelerate training. 

\subsection{Sample Efficiency}
\label{sec: sample efficiency}
\begin{table}[t]
    \caption{Experiments in MuJoCo environments (200k steps, 4 seeds)}
    \label{tab: Experiments in MuJoCo environments}
    \centering
    \begin{tabular}{cccccc}
    \toprule
        \multicolumn{1}{c}{Algorithms}  & \multicolumn{4}{c}{Environments} \\
    \midrule
    MFRL Algorithms  & InvertedPendulum & HalfCheetah &  Hopper & Ant\\
    \midrule
    TD3   &       \textbf{1000.0±0.0} & 7154.8±364.9 &  283.1±245.8  & \textbf{2218.8±539.4}\\
    SAC   &       \textbf{1000.0±0.0} & 9150.5±642.9 &  476.9±297.6  & 492.3±209.4\\
    AE-SAC &    \textbf{1000.0±0.0} & 8703.0±738.4  & 255.6±34.4 & 400.5±112.2\\
    NODA-SAC &    \textbf{1000.0±0.0} & \textbf{9241.8±567.7} &  \textbf{832.5±283.4} & 402.4±54.9\\
    \midrule
    MBRL Algorithms  & Hopper-Stand & Walker-Walk & Finger-Spin &  Walker-Run	
\\
    \midrule
    Dreamer &   198.0±254.3    &   642.9±58.5    &   366.2±252.5    &  \textbf{229.1±11.4}	
\\
    BIRD &   236.4±138.2    &   601.6±37.9    &   361.7±246.8    &  216.0 ±23.4	
\\
    NODA-Dreamer &   \textbf{260.2±284.9}    &   \textbf{764.9±146.9}    &    \textbf{428.5±96.3}   &  224.2±34.7	
\\
    \bottomrule
    \end{tabular}%
\end{table}

\begin{figure}[ht]
\centering
\includegraphics[width=\linewidth]{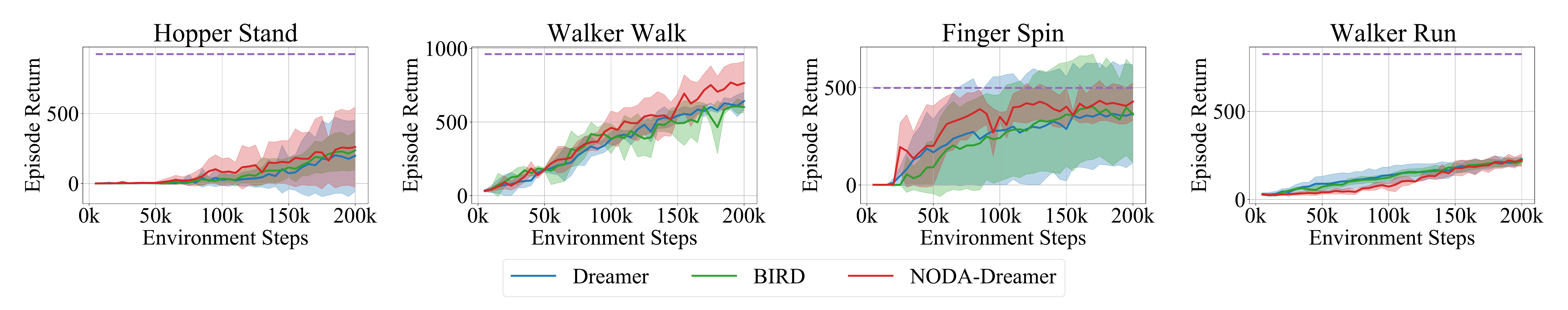}
\caption{Results of Dreamer, BIRD and NODA-Dreamer on DMC. The models are evaluated every 5k steps for 10 episodes, and the means and standard deviations are computed across 10 episodes across 4 seeds. The dashed lines are the asymptotic performances of 5M steps of Dreamer mentioned in ~\citep{hafner2019dream}.}
\vspace{-1.5ex}
\label{fig: dreamer results plots}
\end{figure}
The episode returns achieved by MFRL methods and MBRL methds (with/without NODA) are shown in Table~\ref{tab: Experiments in MuJoCo environments}, Figure~\ref{fig: training process of sample Efficiency part} and Figure~\ref{fig: dreamer results plots}. For MFRL methods, NODA-SAC achieves the best performance with good robustness (see the InvertedPendulum-v2 task in Figure~\ref{fig: training process of sample Efficiency part}). This means that NODA can enhance the performance of the state-of-the-art algorithm SAC by imaginary trajectories. Meanwhile, if we substitute the ODE part using an MLP (denoted as the AE-SAC model), the performance will fall. This suggests that introducing an ODE network, i.e., the prior knowledge of the existence of Hamiltonian indeed helps. But for environments that SAC do not perform well (such as the Ant-v3 task), the enhancement of NODA is not significant. For MBRL methods, NODA-Dreamer can also achieve a higher return than BIRD and Dreamer in many environments. In the task 'Walker Run', during the first half of training, NODA-Dreamer lags far behind Dreamer, but it quickly catches up in the second half. This suggests that NODA has the power to accelerate training after some warm-up steps. Generally, NODA can enhance the performance for both MFRL and MBRL methods. More analyses and experiments can be found in Appendix D.

\section{Conclusion}
This paper proposes a novel simulator called NODA for reinforcement learning. Motivated by Hamiltonian canonical equations in physics, NODA has clear physical meanings. This allows us to incorporate prior knowledge or do transfer learning, which can improve sample efficiency. %Moreover, NODA provides a continuous model for real environments and a different view to look at former models based on a latent space. 
In the theoretical part, we fill the gap between former theorems about Lipschitz models and the physical world by proving that dynamic systems can be Lipschitz and extending former limited theorems to continuous action spaces. Besides, we give uniform transition error bounds and value error bounds for NODA. In the experimental part, we verify that NODA provides a more efficient modeling than HNN, and we can use prior knowledge or transfer learning to further boost its training. Not only is NODA itself sample-efficient, but NODA can improve the sample efficiency of both MFRL and MBRL methods such as SAC and Dreamer. %Specifically, we validate that NODA can enhance the sample efficiency of SAC~\citep{haarnoja2018soft} and Dreamer~\citep{hafner2019dream} within fixed interactions with the environment. 
%\yuhong{and BIRD? do we need to talk about BIRD in the paper, or we just leave it and the corresponding results as supplementary material in the appendix?} 

\section{Ethics Statement}
In this paper, we only use data sets that we are permitted to use, and we present our method and results in a transparent, honest and reproducible way. We fully acknowledge any contributions to this work. The possible harm that our method can bring originates from the modeling errors of NODA, which can lead the agent to go wrong. However, this can be mitigated by theoretical guarantees as well as human supervision. Generally, we believe that our work will make reinforcement learning more accessible by reducing the number of training samples that may be costly to acquire.

\section{Reproducibility Statement}
We give formal proofs for our theorems in Appendix B. Our code is provided in supplementary material. The NODA folder contains experiments about simulation effectiveness and experiments in Gym environments, and the NODA-Dreamer folder contains experiments in DMC environments.

\bibliography{main}
\bibliographystyle{iclr2022_conference}

\appendix
\section{Hamiltonian Canonical Equations: Pendulum-v0 as an Example}
The settings of the Pendulum-v0 task can be found on GitHub\footnote{\url{https://github.com/openai/gym/blob/master/gym/envs/classic_control/pendulum.py}}. 

Let $x$ be the axis horizontally to the right, $y$ be the axis stright up. The pendulum's position can be described by generalized position $\theta$, where $\theta$ is the angle from the $y$ axis to the pendulum.

For the free end, $x = -l\sin\theta$, $y = l\cos\theta$. For a force perpendicular to the pendulum (the angle from the $y$ axis to $u$ is $\theta+\pi/2$) with its 2-norm equaling $u$, its generalized force $Q=\frac{\partial r}{\partial x}u_x + \frac{\partial r}{\partial y}u_y = -l\cos\theta\cdot(-u\cos\theta)-l\sin\theta\cdot(-u\sin\theta)=ul$, which is just the torque of the force.

The kinetic energy $T=1/6ml^2(\dif \theta / \dif t)^2$, and the potential energy $V=1/2mgl\cos\theta$, so we get the Lagrangian in Equation~\eref{eqap: Lagrangian}.
\begin{equation}
\mathcal{L}\left(\theta, \frac{\dif \theta}{\dif t}, t\right) = T-V = \frac 1 6ml^2\left(\frac{\dif \theta}{\dif t}\right)^2 - \frac 12mgl\cos\theta
\label{eqap: Lagrangian}
\end{equation}
We denote $\theta$ as $q$, the canonical momentum for $\theta$ as $p$, then
\begin{equation}
p = \frac{\partial\mathcal{L}}{\partial\left(\frac{\dif q}{\dif t}\right)} = \frac 1 3ml^2\frac{\dif q}{\dif t}
\label{eqap: Canonical momentum}
\end{equation}
Then we can write the Hamiltonian in Equation~\eref{eqap: Hamiltonian}.
\begin{equation}
\mathcal{H}(q, p, t) = \frac{\dif q}{\dif t}p - \mathcal{L} = \frac{3p^2}{2ml^2} + \frac 12mgl\cos q
\label{eqap: Hamiltonian}
\end{equation}
Then Hamiltonian canonical equations give out the dynamics, as is shown in Equation~\eref{eqap: Hamiltonian canonical equations}.
\begin{equation}
    \left\{
    \begin{aligned}
        \frac{\dif q}{\dif t} &= \frac{\partial\mathcal{H}}{\partial p}=\frac{3p}{ml^2}\\
        \frac{\dif p}{\dif t} &= -\frac{\partial\mathcal{H}}{\partial q} + Q(t)=\frac 12mgl\sin q+u(t)
    \end{aligned}
    \right.
    \label{eqap: Hamiltonian canonical equations}
\end{equation}
This is the general case, but the code implementation is a little different. In the code, $u(t)$ is clipped to $[-2, 2]$, $\theta$ is clipped to $[-\pi, \pi)$, and $\dif \theta/\dif t$ is clipped to $[-8, 8]$. This exactly gives the bounds of the state space, since states are $[\cos\theta, \sin\theta, \dif \theta/\dif t]$. The clips can be added to Hamiltonian canonical equations, so the code implementation can be viewed as a solver for Equation~\eref{eqap: Modified Hamiltonian canonical equations}.
\begin{equation}
    \left\{
    \begin{aligned}
        \frac{\dif q}{\dif t} &=\frac{3p}{ml^2}\\
        \frac{\dif p}{\dif t} &=
        \left\{
        \begin{aligned}
        &\min(\frac 12mgl\sin q+\text{clip}(u(t), -2, 2), 0), \text{ if } p>\frac{8}{3ml^2}\\
        &\max(\frac 12mgl\sin q+\text{clip}(u(t), -2, 2), 0), \text{ if } p<-\frac{8}{3ml^2}\\
        &\frac 12mgl\sin q+\text{clip}(u(t), -2, 2), \text{ otherwise}
        \end{aligned}
        \right.
    \end{aligned}
    \right.
    \label{eqap: Modified Hamiltonian canonical equations}
\end{equation}
The encoder $f$ is a function from $[\cos\theta, \sin\theta, \dif \theta/\dif t]$ to $[\theta, 1/3 ml^2\cdot\dif \theta/\dif t]$, which is invertible and Lipschitz. Similarly, the decoder $g$ is Lipschitz. Besides, the right hand side of Hamiltonian canonical equations is continuous and bounded. %So, it actually satisfies most of the conditions in Theorem~\ref{th: Lipschitz dynamic systems}.

\section{Proofs}
\begin{theorem}
\textbf{(Lipschitz Dynamic Systems)} For a dynamic system with a $C^2$ continuous Hamiltonian $\mathcal{H}: \mathbb{R}^{2K+1}\rightarrow\mathbb{R}$, if the state $\bm{s}$ is in a bounded closed set $\mathcal{S}\subset \mathbb{R}^{l}$, the evolution time equals $\tau$, the generalized force $Q_k$ is $C^1$ continuous with respect to states and bounded (for any dimension $k$), and the transformation from states to canonical states $f^*: \mathcal{S}\rightarrow\mathbb{R}^{2K}$ is Lipschitz, then the canonical states are Lipschitz with respect to time, and the environment with respect to canonical states is Lipschitz. Additionally, if the transformation from canonical states to states $g^*: \mathbb{R}^{2K}\rightarrow\mathcal{S}$ is Lipschitz, then the environment is Lipschitz, which means (here $\bm s\neq \bm s'$)
\begin{equation}
\sup_{\bm a\in \mathcal{A}}\sup _{\bm s, \bm s' \in \mathcal{S}} \frac{d_{\mathcal{S}}\left(\bm s_{\text{next}}, \bm s_{\text{next}}'\right)}{d_{\mathcal{S}}\left(\bm s, \bm s'\right)} < \infty.
\label{eqap: Dynamic systems as a Lipschitz model}
\end{equation}
\label{thap: Lipschitz dynamic systems}
\end{theorem}
\begin{proof}
We know that $\bm u = f^*(\bm s)$, where $\bm u$ is the concatenation of $\bm q$ and $\bm p$. Because $\bm s$ is in a bounded closed set, and function $f^*$ is continuous (a Lipschitz function is continuous), each dimension of $f^*$ has a maximum value and a minimum value. So  each dimension of $f^*$ is bounded, which means each dimension of $\bm q$ and $\bm p$ is bounded.

Now we are going to look into Hamiltonian canonical equations. Because $\mathcal{H}$ is $C^2$ continuous, $\frac{\partial\mathcal{H}}{\partial p_k}$ and  $-\frac{\partial\mathcal{H}}{\partial q_k}$ are $C^1$ continuous (for each $k\in\{1,\cdots,K\}$). Because $q_k$ and $p_k$ are bounded for all the values of $k$, the concatenation of $\bm p$, $\bm q$ and $t$ lies in a bounded closed set. So $\frac{\partial\mathcal{H}}{\partial p_k}$ and $-\frac{\partial\mathcal{H}}{\partial q_k}$ have a maximum value and a minimum value, which means they are bounded. Note that $Q_k$ is also bounded, then the right hand side of Hamiltonian canonical equations is bounded.

We denote the bound of $\left|\frac{\partial\mathcal{H}}{\partial p_k}\right|$ as $M_k$, then we can give the Lipschitz constant for $q_k$:
\begin{equation}
|q_k(t_2) - q_k(t_1)| = \left|\frac{\partial\mathcal{H}}{\partial p_k}\bigg|_{t=\xi}(t_2 - t_1)\right| \leq M_k|t_2 - t_1|.
\label{eqap: Bound for q_k}
\end{equation}
Here $\xi$ is a value between $t_1$ and $t_2$.

Similarly, we denote the bound of $\left|-\frac{\partial\mathcal{H}}{\partial p_k} + Q_k\right|$ as $M_k'$, then we can give the Lipschitz constant for $p_k$:
\begin{equation}
|p_k(t_2) - p_k(t_1)| = \left|\left(-\frac{\partial\mathcal{H}}{\partial p_k} + Q_k\right)\bigg|_{t=\xi}(t_2 - t_1)\right| \leq M_k'|t_2 - t_1|.
\label{eqap: Bound for p_k}
\end{equation}
From Equation~\eref{eqap: Bound for q_k} and Equation~\eref{eqap: Bound for p_k}, we know that the canonical states are Lipschitz with respect to time. Let the right hand side of Hamiltonian canonical equations be $\bm I(\bm u, t)$. Since $\dif \bm I(\bm u, t)/\dif \bm u$ is continuous, and $\bm u$ and $t$ lie in a bounded closed set, $\dif \bm I(\bm u, t)/\dif \bm u$ is bounded. Then for different initial canonical states $\bm u_{t_0}$ and $\bm u_{t_0'}'$, we have (here the absolute value and the supremum for a vector/matrix is taken over each element of a vector/matrix, and the $\leq$ between two vectors means the relation between each element):
\begin{equation}
\begin{aligned}
\left|\frac{\dif (\bm u_{t_0}(t) - \bm u_{t_0'}'(t))}{\dif t}\right| &= |\bm I(u_{t_0}(t), t) - \bm I(u_{t_0'}'(t), t)|\\
&\leq\sup_{\bm u, t}\left|\frac{\dif \bm I(\bm u, t)}{\dif \bm u}\right||u_{t_0}(t)-u_{t_0'}'(t)|\\
&\leq \left|\left|\sup_{\bm u, t}\left|\frac{\dif \bm I(\bm u, t)}{\dif \bm u}\right|\right|\right|_\infty|u_{t_0}(t)-u_{t_0'}'(t)|.\\
\end{aligned}
\label{eqap: Bound for evolution of canonical states}
\end{equation}
So 
\begin{equation}
d_\infty(\bm u_{t_0}(t+\tau) , \bm u_{t_0'}'(t+\tau)) \leq \exp\left\{\tau\left|\left|\sup_{\bm u, t}\left|\frac{\dif \bm I(\bm u, t)}{\dif \bm u}\right|\right|\right|_\infty\right\}\cdot d_\infty(\bm u_{t_0}, \bm u_{t_0'}').
\label{eqap: Bound for u (dynamic system)}
\end{equation}
Since $\frac{\dif \bm I(\bm u, t)}{\dif \bm u}$ is bounded, $\exp\left\{\tau\left|\left|\sup_{\bm u, t}\left|\frac{\dif \bm I(\bm u, t)}{\dif \bm u}\right|\right|\right|_\infty\right\}<\infty$. This means the canonical states are Lipschitz with respect to initial canonical states. Because the composition of two Lipschitz functions is Lipschitz and $f^*$ is Lipschitz, we know that the environment is Lipschitz with respect to canonical states.

If we further assume that $g^*$ is Lipschitz, we know the environment is Lipschitz uniformly for all the action $\bm a$ with  respect to states (the composition of two Lipschitz functions is Lipschitz), which is just Equation~\eref{eqap: Dynamic systems as a Lipschitz model}. This concludes the proof.
\end{proof}

\begin{theorem}
\textbf{(Lipschitz NODA)} For the NODA model, if the state $\bm{s}$ is in a bounded closed set $\mathcal{S}\subset \mathbb{R}^{l}$, $f: \mathcal{S}\rightarrow\mathbb{R}^{2K}$ is Lipschitz, the evolution time equals $\tau$, the action $a_m$ is $C^1$ continuous with respect to states and bounded (for any dimension $m$), function $h$ is $C^1$ continuous, then canonical states are Lipschitz with respect to time, and NODA with respect to canonical states is Lipschitz. Additionally, if the transformation from canonical states to states $g: \mathbb{R}^{2K}\rightarrow\mathcal{S}$ is Lipschitz, then NODA is Lipschitz, which means (here $\bm s\neq \bm s'$)
\begin{equation}
\sup_{\bm a\in \mathcal{A}}\sup _{\bm s, \bm s' \in \mathcal{S}} \frac{d_{\mathcal{S}}\left(\bm s_{\text{next}}, \bm s_{\text{next}}'\right)}{d_{\mathcal{S}}\left(\bm s, \bm s'\right)} < \infty.
\label{eqap: NODAs as a Lipschitz model}
\end{equation}
\label{thap: Lipschitz NODA}
\end{theorem}
\begin{proof}
The proof is similar to the proof of Theorem~\ref{thap: Lipschitz dynamic systems}, and here we just give a sketch. We know $\bm u = f(\bm s)$, where $\bm s$ is in a bounded closed set. Because function $f$ is continuous, each dimension of $\bm q$ and $\bm p$ is bounded.

Because $q_k$ and $p_k$ are bounded for all the values of $k$, the concatenation of $\bm q$, $\bm p$, $t$ and $\bm a$ lies in a bounded closed set. As a result, the output of the continuous function $h$ is bounded. We denote the bound of the $k$th output of the right hand side of the ODE as $W_k$. 
\begin{equation}
|u_k(t_2) - u_k(t_1)| \leq W_k|t_2 - t_1|
\label{eqap: Bound for u_k}
\end{equation}
Equation~\eref{eqap: Bound for u_k} tells us that $\bm{u}$ is Lipschitz with respect to time $t$. 

After that, we can just follow the corresponding part (the part after Equation~\eref{eqap: Bound for p_k}) in Theorem~\ref{thap: Lipschitz dynamic systems}'s proof to get final results. This concludes the proof.
\end{proof}

With this bound of multi-step transition, we can give bounds for state values. These bounds tell us that under certain conditions, the optimal policy learned from the simulator and the optimal policy learned from the environment do not differ much. 

\begin{theorem}
\textbf{(Transition Error Bounds)} Under the conditions in Theorem~\ref{thap: Lipschitz dynamic systems} and Theorem~\ref{thap: Lipschitz NODA}, we already know that the transition function $T_{\mathcal G}\left(\bm s^{\prime} \mid \bm s, \bm a\right)$ of the environment and the transition function $\widehat T_{\mathcal G}\left(\bm s^{\prime} \mid \bm s, \bm a\right)$ of the NODA model is Lipschitz. We denote the Lipschitz constant of these transition functions as $K_1$ and $K_2$. Let $\bar{K} = \text{min}\{K_1, K_2\}$, then $\forall n\geq 1$:
\begin{equation}
\delta(n):=W\left(\widehat T_{\mathcal{G}}^{n}(\cdot \mid \mu), T_{\mathcal{G}}^{n}(\cdot \mid \mu)\right) \leq \Delta \sum_{i=0}^{n-1}(\bar{K})^{i}.
\label{eqqp: Transition n-step bounds}
\end{equation}
Here $\Delta$ is defined as the upper bound of Wasserstein metric between $\widehat T\left(\cdot \mid \bm s, \bm a\right)$ and $T\left(\cdot \mid \bm s, \bm a\right)$.
\label{thap: Transition error bounds}
\end{theorem}

The original theorem~\citep{asadi2018lipschitz} gives a a bound for a fixed action sequence. However, here our definitions of Lipschitz environments and models are uniform for all actions. So, by using the original proof, we give a same error bound for all possible action sequences. Thus, we get a uniform error bound under the continuous action space. This concludes the proof. Now we are going to give bounds for state values.

\begin{theorem}
\textbf{(Value Error Bounds)} Under all the conditions in Theorem~\ref{thap: Transition error bounds}, if the reward function $R(\bm s)$ is (uniformly) Lipschitz, which means we can define
\begin{equation}
K_R:= \sup_{a\in \mathcal{A}} \sup _{\bm s_{1}, \bm s_{2} \in \mathcal{S}} \frac{\left|R\left(\bm s_{1}, \bm a\right)-R\left(\bm s_{2}, \bm a\right)\right|}{d_\mathcal{S}\left(\bm s_{1}, \bm s_{2}\right)} < \infty.
\label{eqap: Lipschitz reward function}
\end{equation}
If we define state values as
\begin{equation}
V_{T}(\bm s):=\sum_{n=0}^{\infty} \gamma^{n} \int T_{\mathcal{G}}^{n}\left(\bm s^{\prime} \mid \delta_{\bm s}\right) R\left(\bm s^{\prime}\right) d \bm s^{\prime},
\label{eqap: State values}
\end{equation}
where $\delta_{\bm s}$ means the probability that the state is $\bm s$ equals 1.
Then $\forall \bm s\in\mathcal{S}$ and $\bar{K}$ (defined in Theorem~\ref{thap: Transition error bounds})$\in[0, \frac{1}{\gamma}]$, we have
\begin{equation}
\left|V_{T}(\bm s)-V_{\widehat{T}}(\bm s)\right| \leq \frac{\gamma K_{R} \Delta}{(1-\gamma)(1-\gamma \bar{K})}.
\label{eqap: Value error bounds}
\end{equation}
\label{thap: Value error bounds}
\end{theorem}

The original theorem~\citep{asadi2018lipschitz} is for action space $\mathcal{A} = \{\bm a\}$, which means there is only one action. Besides, the original theorem assumes that the reward function only depends on state $\bm s$.

However, these strict limitations can be overcome by our uniformly Lipschitz models and requiring a uniformly Lipschitz reward function, as is shown in Equation~\eref{eqap: Lipschitz reward function}. As long as these conditions are satisfied, we can just follow the path of the original proof. Specifically, for any action sequence, Equation~\eref{eqap: Value error bounds} holds, so we get a uniform error bound. This concludes the proof.

Here we can find that a crucial condition in the proofs of Theorem \ref{thap: Transition error bounds} and Theorem~\ref{thap: Value error bounds} is that both the model and the environment is Lipschitz. In fact, Theorem \ref{thap: Lipschitz dynamic systems} and Theorem~\ref{thap: Lipschitz NODA} lay the foundations of uniform transition error bounds and value error bounds.

\section{Algorithm Details}
Algorithm~\ref{alg: NODA-SAC},~\ref{alg: NODA-SAC Interaction} and~\ref{alg: NODA-SAC Training} show how we combine SAC~\citep{haarnoja2018soft} with NODA. Generally, we use interactions with the environment to train a NODA model, and use it to generate imaginary trajectories to facilitate the training of SAC. We combine AE with SAC in the same way. For efficiency, we run 50 interactions with the real environment, and then update the agent and our model 50 times, which is also implemented by spinning up~\citep{SpinningUp2018}. We also reduce the number of imaginary training batches per epoch when interactions with the real environment is sufficient.

Combining NODA with MBRL methods is easier. For Dreamer~\citep{hafner2019dream}, it itself works in an auto-encoder's fashion, that is, it encodes observations and actions into compact latent states and makes decisions and predictions in the latent state space using its world model, after which the observations are reconstructed via representation learning. Therefore, in order to combine NODA with Dreamer, we just need to replace the deterministic path(a GRU cell) of the recurrent state space model(RSSM) in Dreamer with an ODE network, and we get NODA-Dreamer. The input compact state to the original deterministic state model in Dreamer, which is believed to be a canonical state, evolves by going through the ODE network. We also similarly implement NODA-BIRD by replacing the deterministic path of the RSSM in BIRD~\citep{zhu2020bridging} with an ODE network, since BIRD and Dreamer have the same RSSM. All other algorithm details remain the same as in Dreamer and in BIRD respectively.

%Since Dreamer itself works in an auto-encoder's fashion, we can easily get NODA-Dreamer by replacing the deterministic path(a GRU cell) of the latent state space model in Dreamer with an ODE network. The input to the original deterministic path in Dreamer, which is assumed to be the canonical state, evolves by going through the ODE network.

\begin{figure}[ht]
\centering
\includegraphics[width=0.5\linewidth]{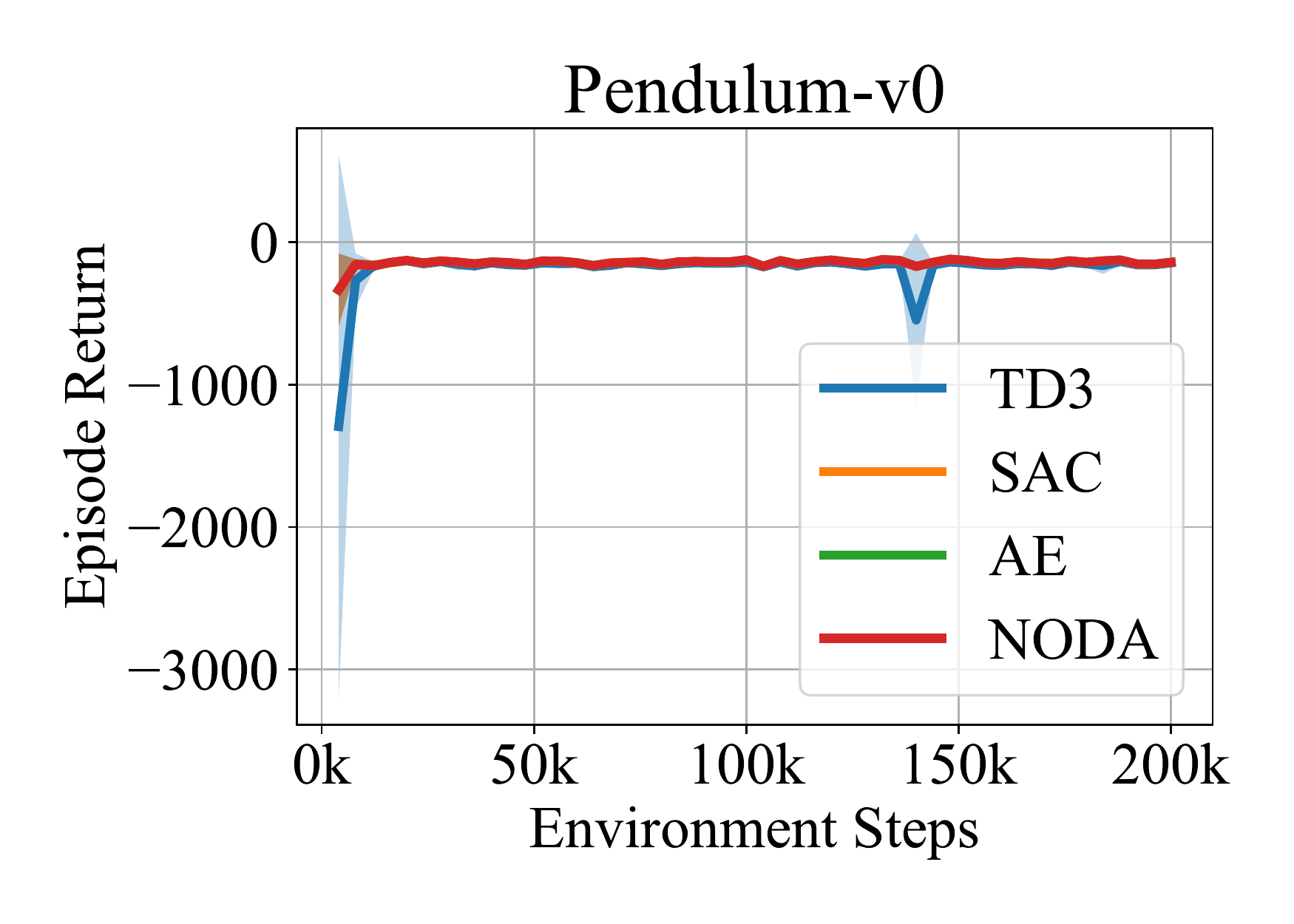}
\caption{Results of TD3, SAC, AE and NODA on the Pendulum-v0 environment in Gym. The models are evaluated every 4k steps for 10 episodes, and the means and standard deviations are computed across 10 episodes across 4 seeds. The performance of NODA-SAC is stable, and NODA is able to facilitate the training of SAC in an early stage.}
\label{fig: pendulum-v0}
\end{figure}

\begin{algorithm}[tb]
  \caption{NODA-SAC for Reinforcement Learning}
  \label{alg: NODA-SAC}
\begin{algorithmic}
  \STATE {\bfseries Input:} SAC actor $\pi$, SAC critics $q_1$, $q_2$, environment $env$, NODA model $m$ (with optimizer $opt_m$), warmup steps $N_1$, model generate intervals $N_2$, model planning steps $N_3$, interaction buffer $N_4$, batch size $B_1$, model batch size $B_2$, SAC training function $T_1$, NODA training function $T_2$, data generation function for NODA $G$
  \STATE $done\leftarrow True$
  \STATE $step\leftarrow 0$
  \STATE $D\leftarrow [~]$
  \STATE $D_m\leftarrow [~]$
  \REPEAT
  \IF{$done$}
  	\STATE $s\leftarrow env.reset()$
  \ENDIF
  
  \IF{$step \leq N_1$}
  	\STATE $a\leftarrow env.sample(B_1)$
  	\STATE $s_2, r, done\leftarrow env.step(a)$
  	\STATE $D.append(\{s, a, s_2, r, done\})$
  \ELSE
  	\STATE $a\leftarrow \pi(s)$
  	\STATE $s_2, r, done\leftarrow env.step(a)$
  	\STATE $D.append(\{s, a, s_2, r, done\})$
  	\STATE $batch\leftarrow D.sample(B_1)$
  	\STATE $\pi, q_1, q_2\leftarrow T_1(\pi, q_1, q_2, batch)$
  	\STATE $m\leftarrow T_2(m, batch, opt_m)$
  	\IF{$step \% N_2 == 0$}
        \STATE $D_m\leftarrow G(\pi, q_1, q_2, m, D, D_m, B_2, N_3)$
        \STATE $batch\leftarrow D.sample(B_1)$
        \STATE $\pi, q_1, q_2\leftarrow T_1(\pi, q_1, q_2, batch)$
    \ENDIF
  \ENDIF

  \STATE $step\leftarrow step + 1$
  \UNTIL{$step = N_4$}
  \RETURN SAC actor $\pi$, SAC critics $q_1$, $q_2$, NODA model $m$
\end{algorithmic}
\end{algorithm}

\begin{algorithm}[tb]
  \caption{Interaction function $G$ in NODA-SAC}
  \label{alg: NODA-SAC Interaction}
\begin{algorithmic}
  \STATE {\bfseries Input:} SAC actor $\pi$, SAC critics $q_1$, $q_2$, NODA model $m$, environment buffer $D$, model buffer $D_m$, batch size $B$, model planning steps $N$, learning rate $lr$ (default=1e-5)
  \STATE $batch\leftarrow D.sample(B)$
  \STATE $s, a\leftarrow batch.s, batch.a$
  %\STATE $opt_a = ADAM(a, lr)$
  \STATE $planning\_ step\leftarrow 0$
  \REPEAT
% 	  \FOR{$i~in~range(5)$}
% 	  \STATE $loss_a=-min(q_1(s,1), q_2(s,a))$
% 	  \STATE $opt_a.zero\_grad()$
% 	  \STATE $loss_a.backward()$
% 	  \STATE $opt_a.step()$
% 	  \ENDFOR
	  \STATE $s_2, r, \text{done} \leftarrow m.\text{step}(s, a)$
	  \STATE $D_m.append(\{s, a, s_2, r, \text{done}\})$ 
	  \STATE $s\leftarrow s_2$
	  \STATE $a\leftarrow \pi(s)$
	  \STATE $planning\_ step\leftarrow planning\_ step + 1$
  \UNTIL{$planning\_ step = N$}
  \RETURN model buffer $D_m$
\end{algorithmic}
\end{algorithm}

\begin{algorithm}[tb]
  \caption{Training function $T_2$ for NODA}
  \label{alg: NODA-SAC Training}
\begin{algorithmic}
  \STATE {\bfseries Input:} NODA model $m$, a batch of buffer $batch$, optimizer $opt$
 
  \STATE $s, a, s_2, r\leftarrow batch.s, batch.a, batch.s_2, batch.r$
  \STATE $loss_{srecon}\leftarrow (||m.decoder(m.encoder(s)) - s||_2^2).mean()$
  \STATE $m.set\_state(s)$
  \STATE $s_2', r' \leftarrow m.step(a)$
  \STATE $loss_{spred}\leftarrow (||s_2' - s||_2^2).mean()$
  \STATE $loss_s\leftarrow loss_{srecon} + loss_{spred}$
  \STATE $loss_r\leftarrow (||r'- r||_2^2).mean()$
  \STATE $loss \leftarrow 0.5\cdot loss_s + 0.5\cdot loss_r$
  \STATE $opt.zero\_grad()$
  \STATE $loss.backward()$
  \STATE $opt.step()$
  \RETURN NODA model $m$
\end{algorithmic}
\end{algorithm}

\section{Experimental Details}
In the pixel pendulum task, we run all the models for 3,000 batches with a batch size of 200. we decode the output of $h$ as the velocity prediction and set the next state to the evolution of the current state after a very short time (assuming the velocity does not change) as a weak supervision. More details can be found in our code. In the real pendulum task, we run all the models for 2,000 batches with a batch size of 200. The learning rate for both tasks is 1e-3.

When comparing different NODA models in the Ant-v3 task, we generate 20,000 steps for training and 20,000 steps for testing by a random policy. Each NODA model is trained by 3,000 batches. The batch size is 256, and the learning rate is still 1e-3. 

For TD3~\citep{fujimoto2018addressing}, SAC~\citep{haarnoja2018soft}, AE-SAC and NODA-SAC, we use the code provided by spinning up~\citep{SpinningUp2018} (the PyTorch~\citep{NEURIPS2019_9015} version). We modify it to use GPU, and use batch size=100 for TD3 and SAC (original setting), and batch size=256 for AE-SAC and NODA-SAC. Other parameters are the same as what are used in spinning up. We implement NODA by PyTorch, and we use torchdiffeq\footnote{\url{https://github.com/rtqichen/torchdiffeq}} as the implementation of the ODE network~\citep{chen2018neural, chen2021eventfn}.

In addition to MuJoCo~\citep{todorov2012mujoco, schulman2015high} environments, we also compare these algorithms over a simple physical environment, Pendulum-v0 in Gym~\citep{brockman2016openai}, and the result is shown in Figure~\ref{fig: pendulum-v0}. It shows that the return of NODA-SAC converges quickly, and is quite stable.

For NODA-Dreamer, all the experimental details remain the same as those in Dreamer, except that we run only 200,000 steps for every task we try, and we evaluate the performance of the model every 5,000 steps. What's more, we use tfdiffeq\footnote{\url{https://github.com/titu1994/tfdiffeq}} as the implementation of the Neural ODE, which runs entirely on Tensorflow Eager Execution. We evaluate Dreamer, BIRD and NODA-Dreamer on 6 tasks in total in Figure~\ref{fig: noda_dreamer_results}. We also evaluate NODA-BIRD, which is mentioned in the previous section, on 3 tasks, as shown in Figure~\ref{fig: noda_dreamer_bird_results}.

\begin{figure}[ht]
\centering
\includegraphics[width=\linewidth]{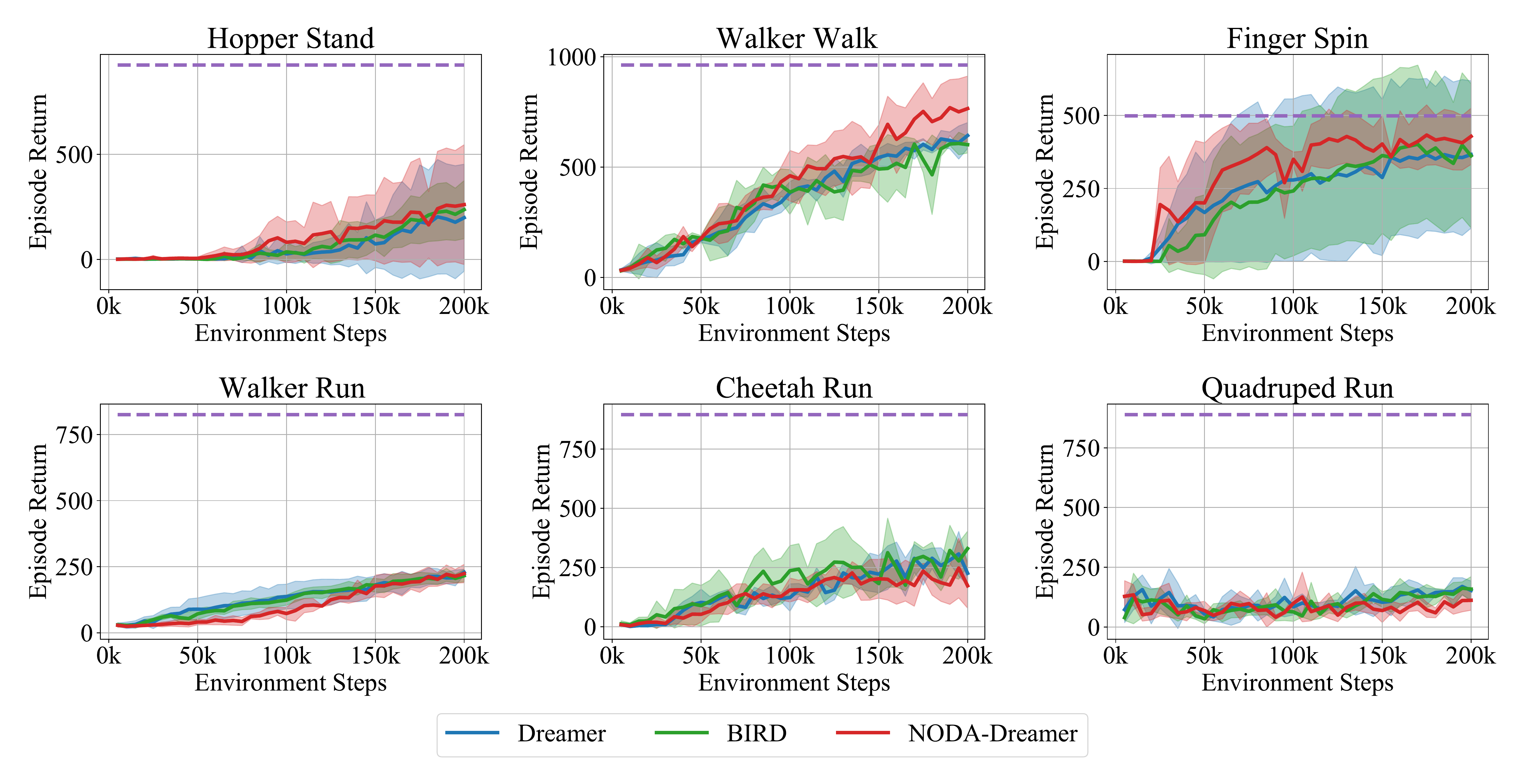}
\caption{Results of Dreamer, BIRD and NODA-Dreamer on 6 tasks in DeepMind Control Suite. The models are evaluated every 5k steps for 10 episodes, and the means and standard deviations are computed across 10 episodes across 4 seeds. The dashed lines are the asymptotic performances of 5M steps of Dreamer mentioned in ~\citep{hafner2019dream}. In many tasks, NODA-Dreamer has higher episode return starting early in the training process, which shows that NODA can assist in improving sample efficiency. In "Walker Run", during the first half of training, NODA-Dreamer lags far behind Dreamer and BIRD, but it quickly catches up in the second half. This suggests that NODA has the power to accelerate training after some warm-up steps. In "Quadruped Run", 200,000 steps are not sufficient for any algorithm to learn meaningful information.}
\label{fig: noda_dreamer_results}
\end{figure}

\begin{figure}[ht]
\centering
\includegraphics[width=\linewidth]{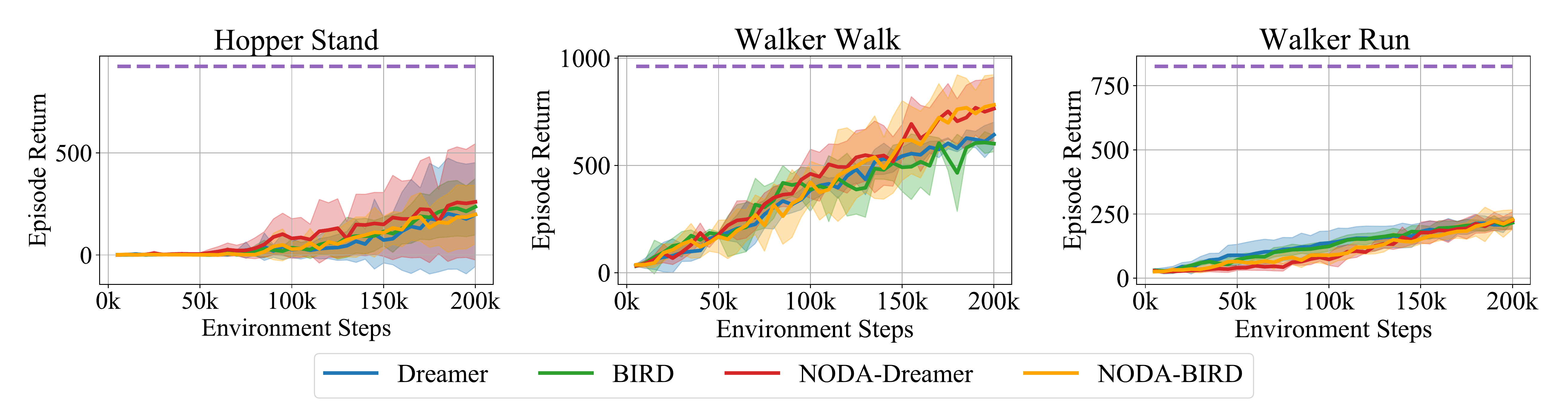}
\caption{Results of Dreamer, BIRD, NODA-Dreamer and NODA-BIRD on 3 tasks in DeepMind Control Suite. The models are evaluated every 5k steps for 10 episodes, and the means and standard deviations are computed across 10 episodes across 4 seeds. The dashed lines are the asymptotic performances of 5M steps of Dreamer mentioned in ~\citep{hafner2019dream}. The results suggest that NODA is very promising, in the sense that it can be easily combined with different MBRL methods, especially auto-encoder-type methods, by simply modifying their transition model.}
\label{fig: noda_dreamer_bird_results}
\end{figure}

\end{document}